\documentclass[twoside,11pt]{article}

\usepackage{blindtext}

\usepackage[abbrvbib, preprint]{jmlr2e}
\usepackage{amsmath}
\usepackage{tikz}

\newtheorem{assumption}[theorem]{Assumption}

\usepackage{lastpage}

\firstpageno{1}

\begin{document}

\title{Understanding Data Augmentation under Label Noise}

\author{\name Hong Zheng\thanks{This work is under peer review, which may contain some minor errors in mathematical proofs.} \email hirozheng@my.swjtu.edu.cn \\
       \addr School of Computing and Artificial Intelligence\\
       Southwest Jiaotong University\\
       Chengdu, 611756, China 
       \AND
       \name Fei Teng \email fteng@swjtu.edu.cn \\
       \addr School of Computing and Artificial Intelligence\\
       and Engineering Research Center of Sustainable Urban Intelligent Transportation\\
       Southwest Jiaotong University  \\
       and Ministry of Education \\
       Chengdu, 611756, China 
       }

\editor{My editor}

\maketitle

\begin{abstract}
Data augmentation (DA) has become a fundamental component of modern machine learning systems and is widely regarded as an effective strategy for improving robustness and generalization. Despite its empirical success, the theoretical impact of DA under label noise remains poorly understood. In this work, we first theoretically explain when DA is beneficial and when it becomes detrimental, and provide sufficient conditions for distinguishing between these two regimes under both the i.i.d. learning setting and the more challenging domain generalization scenario. Then, we reveal a counterintuitive phenomenon: DA essentially acts as an uncertainty amplifier in the learning process, decreasing model stability and generalization performance. This finding is supported by the proposed learning and generalization bounds. Our analyses do not depend on specific DA techniques and require only mild assumptions, indicating that the results are applicable to a broad range of machine learning scenarios. Namely, we provide a unified theoretical foundation for understanding, analyzing, and designing DA methods in the presence of label noise. 
\end{abstract}


\section{Introduction}
\label{sec1}

Data augmentation (DA) techniques are widely applied in machine learning and are typically considered as regularization terms, aiming to increase the robustness of models for prediction~\citep{b17}. However, it is often unclear exactly how, why, and when DA methods work for given data~\citep{b11}. In particular, for real-world applications, label noise is common, and its presence harms the learning process~\citep{b13}. Under such data conditions, answering the above questions is essential for understanding the underlying mechanisms of machine learning and for facilitating practical applications. 

\cite{b13} conducted a series of experiments on DA methods under noisy-label conditions for empirical understanding. They stated that the appropriate selection of DA can drastically improve model robustness to label noise and argued that the best augmentation strategy appears to be dataset-dependent. However, several aspects of their results require more careful scrutiny. For example, when the noise types are symmetric or asymmetric, DA methods can substantially improve the evaluation performance even when the noise rate in the training set is as high as $0.8$. In contrast, when the noise type is semantic noise, the improvement is not significant, and in some cases, performance even degrades. Note that a noise rate of $0.8$ is typically considered highly challenging for machine learning, since successful learning is generally expected only when the label noise rate is less than $0.5$. Nevertheless, their experiments clearly showed that DA methods do not always improve performance; in most cases, they degrade it. 

In this work, we present theoretical analyses to understand DA methods under noisy-label conditions in both i.i.d. and Domain Generalization (DG) learning settings. Specifically, we theoretically explain when DA is beneficial and when it becomes detrimental, and provide sufficient conditions for distinguishing between these two regimes. This exactly provides a theoretical explanation for the experimental observations in~\citep{b13}. We then extend these results to the more challenging DG scenario, although the obtained conditions are more restrictive than those in the i.i.d. setting, which in turn reflects the inherent challenges of DG. Additionally, we reveal a counterintuitive phenomenon: DA essentially acts as an uncertainty amplifier in the learning process, decreasing model stability and generalization performance. This finding is supported by our proposed learning and generalization bounds, as DA methods make these bounds looser. 

In conclusion, the main contributions of our work are as follows:
\begin{enumerate}
\item[(1)]{ We establish a theoretical framework for analyzing DA under label noise in both i.i.d. and independent and non-identically distributed (i.n.d.) settings. Within this framework, we derive sufficient conditions under which DA fails. We use simple and general models to represent various types of label noise and DA methods. The introduced assumptions are mild in most machine learning scenarios, and the lemmas used in the proofs are simple and classical. Moreover, the presented theorems establish a theoretical foundation for developing future DA methods under noisy-label settings. 
}
\item[(2)]{ We reveal that DA increases uncertainty in both learning and generalization, regardless of the existence of label noise. It is well known that machine learning is commonly studied under the framework of probably approximately correct (PAC) learning theory. This implies that the learning process is inherently uncertain, and the use of DA methods may further amplify this uncertainty, potentially leading to significant risks in high-risk application domains, such as medical diagnosis, autonomous driving, and financial decision-making. 
}
\end{enumerate}

\section{Related Works}
\label{sec2}

In this section, we briefly discuss the development of DA methods and related empirical and theoretical analyses, providing an initial glance at studies related to DA. For more details, please refer to the cited references. 

Initially, DA methods are introduced in the field of computer vision~\citep{b18}, and currently many DA methods are still mainly focused on computer vision tasks, such as image flipping, rotation, Mixup series~\citep{b14, b15}, etc. Some methods based on generative adversarial networks (GANs) are also introduced, such as DA-GANs~\citep{b19}. Moreover, nowadays, methods based on large models are also presented~\citep{b20}. GAN-based and large-model-based methods typically assume that the generated $x’_i$ is independent of $x_i$ for $i \in [n]$, where $n$ denotes the number of data samples under the i.i.d. assumption. Thus, they typically consider that their methods expand the amount of data for learning. However, considering a non-independent assumption is more natural~\citep{b17}. Thus, in our work, we also consider the non-independent setting and assume that $x’_i$ is independent of $x_j$ for any $i \ne j \in [n]$. Namely, as long as the augmented inputs are independent across different training samples, classical learning-theoretic analysis remains applicable. Therefore, we do not regard DA methods as increasing the effective number of training samples. Additionally, there are also some studies that focus on DA methods under the learning scenario of domain generalization or out-of-domain settings, such as~\citep{b23, b22}. 

Along the development of DA, learning with noisy labels has gained increasing attention, as most real-world applications suffer from such data conditions~\citep{b21, b29}. Consequently, some studies also consider using DA methods for learning robust models under label noise, such as~\citep{b24,b25}. Moreover, some studies suggest that Mixup-based methods can benefit learning under noisy data~\citep{b16}. One particular work~\citep{b13} has recently attracted our attention; although published several years ago, it systematically conducted a series of experiments on DA methods under label noise to empirically investigate their behavior. However, the study did not provide theoretical explanations for its findings, and several experimental results lacked sufficient justification. 

Theoretical analysis in machine learning is also important for its development, as it aims to explain learning phenomena in a direct manner, namely through mathematical analysis. Similarly, several studies have attempted to provide theoretical explanations for DA methods. For example, existing studies have sought to explain the effectiveness of DA using different theoretical frameworks, including Markov processes~\citep{b26}, kernel methods~\citep{b27}, and nonlinear regularization~\citep{b28}. \cite{b17} showed that DA is equivalent to an averaging operation over the orbits of a certain group that approximately preserves the data distribution. \cite{b11} found that, on the one hand, DA improves generalization, while on the other hand, it is less forgiving for over-parameterized regression models. These findings suggest that the effectiveness of DA strongly depends on the task, model, and augmentation type. However, the above studies generally do not account for the presence of label noise.

\section{A Classical Learning Perspective}
\label{sec3}

\subsection{Setup}
\label{sec3.1} 

\noindent\textbf{Basic Data Assumption}. We first clarify the basic data structure for subsequent analysis. Under the classical supervised learning setting, assuming ${\cal X}$ and ${\cal Y}$ being respectively an input and output space, the clean dataset without labeling noise is denoted as 
\begin{align}
D^c=\{(x_1,y_1),\ldots, (x_n,y_n)\}, \nonumber
\end{align}
where each pair $(x,y) \in {\cal X} \times {\cal Y}$ is independently drawn from an unknown joint distribution $P(x,y)$, i.e., $(x,y) \sim_{i.i.d.} P(x,y)$. The dataset $D^c$ of size $n$ is used as the training set for learning models. Correspondingly, a test dataset is assumed to be drawn from the same distribution $P(x,y)$. Due to the i.i.d. assumption, the test dataset is typically omitted from the theoretical analysis. Therefore, we do not introduce specific notation for the test data. 

In this work, we mainly focus on the noisy label condition for data. Thus, we assume the training data 
\begin{align}
D=\{(x_1,\tilde{y_1}),\ldots, (x_n,\tilde{y_n})\}.\nonumber 
\end{align}
Here, we assume that $\tilde y \sim P(\tilde y|x,y)$. For this assumption, we provide several explanations. This labeling noise model $P(\tilde y|x,y)$ is, in fact, a master model that can represent any noisy model depending on how it is constrained. For example, for a binary classification problem, if we consider the symmetric uniform noise model $P(\tilde{y} \ne y|y) = \eta$ and $P(\tilde{y} = y | y) = 1 - \eta$, where $\eta \in [0, 1]$, we have $P(\tilde y|x,y) = P(\tilde y|y)$. If we consider the Massart noise model $P(\tilde y \ne y | x) = \eta(x)$, where $\eta(x) \le \eta_{max} \le 1/2$, we have $P(\tilde y|x,y) = P(\tilde y|x)$. Moreover, the instance-dependent noise model $P(\tilde y|x,y) = P(\tilde y \ne y |x,y)$ can also be considered. 

Based on the assumed labeling noise model, we obtain a new data distribution $P(x,\tilde y)=\sum\nolimits_{y}{P(x,y)P(\tilde y|x, y)}$; if considering the continuous case, $P(x,\tilde y)=\int{P(x,y)P(\tilde y|x, y)dy}$. In fact, we can consider the operator $P(\tilde y|x, y)$ as performing an irreversible transformation on the original distribution $P(x,y)$, which is a convolutional operator. 

\noindent\textbf{Data Augmentation}. Having introduced the noisy labeled data model and its associated distribution, we now consider the data augmentation (DA) model. In practice, two types of DA are commonly used: online and offline DA. Online DA refers to applying augmentation during the training process, whereas offline DA refers to augmenting the training dataset prior to training. 

In this work, we ONLY focus on the online setting in our theoretical analysis, as the offline setting may violate the i.i.d. assumption. For example, for any $i \in [n]$, let $x'_i = T(x_i)$, where $T(\cdot \mid x)$ is a stochastic transformation kernel. It is clear that $x'_i$ depends on $x_i$, introducing dependence between the original and augmented samples. In the online setting, it is sufficient to assume that $x'_i$ is independent of $x_j$ for all $j \ne i \in [n]$. Consequently, there is no need to consider the number of times that $T$ is applied, as repeated applications cannot increase the number of independently generated samples. That is, after applying $T$, the number of data samples in $D$ remains $n$. 

We now model DA as a stochastic transformation kernel $T(\cdot | x)$, i.e.,
\begin{align}
x' \sim T(\cdot | x). \nonumber 
\end{align}
Note that this $T$ can de also considered as a mixture of $m$ kernels, such as $T(x' | x) = \sum_j^m{\pi_j t_j(x' | x)}$, where $\sum_j{\pi_j} = 1$. Since this consideration does not affect the analysis, we omit further discussion of it. Now, based on the DA kernel $T$, we obtain the following distribution 
\begin{align}
P_{DA}(x',\tilde y)=\int{T(x'| x)P(x,\tilde y)dx}. \nonumber
\end{align}
Typically, this distribution is considered a pushforward distribution in most of the literature. In our analysis, we additionally treat it as a fixed distribution obtained by applying $T$ to $P(x,\tilde y)$, where $T$ can also be regarded as a non-orthogonal convolutional kernel. In most training scenarios, $T$ is not $1\text{-}1$, that is, there does not exist an inverse mapping $T^{-1}$ such that $x=T^{-1}(x')$. Consequently, unless $T$ is $1\text{-}1$, $P(x’, \tilde y) \ne P(x, \tilde y)$, which supports the validity of our consideration. Additionally, we assume that for $P$-almost every $(x', \tilde y)$, it holds that $P_{DA}(x',\tilde y) > 0$. This assumption implies that $P(x, \tilde y) \ll P(x', \tilde y)$, allowing the Kullback–Leibler (KL) divergence to be used to measure their discrepancy in the following analysis. If one considers this assumption to be strong, the Jensen–Shannon (JS) divergence can also be used instead.

The kernel $T$ does not affect the label $\tilde y$. If we consider DA methods such as Mixup~\citep{b14}, which introduce new labels, this $T$ is no longer sufficient. Thus, we introduce the following model 
\begin{align}
(x',\tilde y') \sim T'(\cdot | x,\tilde y), \nonumber 
\end{align}
and correspondingly the distribution is 
\begin{align}
P_{DA}(x',\tilde y')=\iint{T'(x',\tilde y'| x, \tilde y)P(x,\tilde y)dxd\tilde y}. \nonumber
\end{align}
Similarly, we assume that $P_{DA}(x',\tilde y') > 0$ for $P$-almost every $(x', \tilde y')$. For notational convenience, we subsequently use $P^c$, $P$, $P_{DA}$, and $P'_{DA}$ to represent $P(x,y)$, $P(x,\tilde y)$, $P_{DA}(x',\tilde y)$, and $P_{DA}(x',\tilde y')$, respectively. 

\noindent\textbf{Others}. We consider a hypothesis space ${\cal H}$ equipped with a measurable structure, and let $h \in {\cal H}: {\cal X} \to {\cal Y}$ denote an arbitrary hypothesis. To quantify the prediction accuracy of $h$, we introduce a cost function $c: {\cal Y} \times {\cal Y} \to {\mathbb R}_{\ge 0}$. Then, the loss of $h$ with respect to any data pair $(x, y)$ is defined as
\begin{align}
\ell(h,(x, y))=c(h(x), y). \nonumber 
\end{align}
Based on this loss, we then introduce a definition of a property of the loss function as follows.
\begin{definition}[$\sigma$-admissible loss function~\citep{b1}] 
\label{defi1}
A loss function $\ell$ defined on ${\cal H} \times \{{\cal X} \times {\cal Y}\}$ is $\sigma$-admissible with respect to ${\cal H}$ if the associated cost function $c$ is convex with respect to its first argument and the following condition holds
\begin{align}
\forall y_1,y_2 \in {\cal D},\forall y \in {\cal Y},\left| c(y_1,y)-c(y_2,y) \right| \le \sigma \left|y_1-y_2\right|, \nonumber
\end{align}
where ${\cal D}=\{y: \exists h \in {\cal H}, \exists x\in{\cal X}, h(x)=y\}$ is the domain of the first argument of $c$. 
\end{definition} 
\noindent\textbf{Remarks}: (a) This definition, in fact, implies that $\ell$ is $\sigma$-Lipschitz. (b) $\|\nabla \ell (\cdot) \|_{\infty} = \sigma$.

For measuring the discrepancy between two distributions, if $P \ll Q$, we use the following divergence. 
\begin{definition}[Kullback–Leibler (KL) divergence]
\label{defi2}
Let $P$ and $Q$ be two probability measures on the same measurable space $({\cal Z}, {\Omega})$, and assume that $P$ is absolutely continuous with respect to $Q$, i.e., $P \ll Q$. The KL-divergence from $P$ to $Q$ is defined as 
\begin{align}
KL(P||Q):=\int\nolimits_{\cal Z}{log\left(\frac{dP}{dQ}(z)\right)dP(z)}. \nonumber
\end{align}
In the discrete case:
\begin{align}
KL(P||Q):=\sum\nolimits_{z\in{\cal Z}}{P(z)log\frac{P(z)}{Q(z)}}. \nonumber 
\end{align}
\end{definition}
\noindent\textbf{Remarks}: (a) $0\cdot \log(0/0)=0$. (b) $\exists a:Q(a)=0,P(a)>0\Rightarrow KL(P||Q)=\infty$. (c) $KL(P||Q)\ge 0; KL(P||Q)=0 \; iff \; P=Q$. (d) $KL(P||Q) \ne KL(Q||P)$. 

We additionally introduce the Wasserstein-1 distance, which characterizes distributional discrepancy through geometric discrepancy. In certain cases, this characterization simplifies the analysis. However, when the conditions required for such a transformation are restrictive, we still employ the KL divergence as the discrepancy measure. 
\begin{definition}[Wasserstein-1($W_1$) Distance]
\label{defi3}
Let $({\cal Z},d)$ be a metric space, and let $P$ and $Q$ be probability measures on ${\cal Z}$. The $W_1$ distance between $P$ and $Q$ is defined as 
\begin{align}
W_1(P,Q):=\mathop{\inf}\limits_{\gamma\in\prod(P,Q)}{{\mathbb E}_{(z_1,z_2)\sim \gamma}[d(z_1,z_2)]}, \nonumber
\end{align}
where $\prod(P,Q)$ denotes the set of all coupling of $P$ and $Q$, and each $\gamma \in \prod(P,Q)$ is a joint distribution whose marginals are $P$ and $Q$, respectively. 
\end{definition}
\noindent\textbf{Remarks}: (a) $W_1(P,Q)\ge0;W_1(P,Q) = 0 \; iff \; P=Q$. (b) $W_1(P,Q)=W_1(Q,P)$. (c) $W_1(P,Q)\le W_1(P,M) + W_1(M,Q)$ for any distribution $M$. (d) $W_1(P,Q) \le Dia$, where $Dia$ is the diameter of the union of the supports.

\subsection{A Warm-up Analysis}
\label{sec3.2}

In this subsection, we provide several interesting results based on a strong but simple assumption. 
We begin by imposing the following assumption on $D$. 
\begin{assumption}[Separability of $D$]
\label{assu0}
Assume that $D$ contains $\eta$-noisy labels such that $n = n_{\eta} + n_c$ and $n_{\eta} \le n_c$. Here, $n_{\eta} $ denotes the number of noisy samples. 
\end{assumption}
\noindent\textbf{Remarks}: Here, we assume that the dataset $D$ can be explicitly partitioned into clean and noisy subsets. In practice, however, this assumption is relatively strong. Nevertheless, this simplified yet strong assumption enables several interesting analyses. Here, $\eta$ is not associated with any specific noise model, and the $n_{\eta}$ noisy samples may arise from arbitrary sources; $\eta$ only serves as a criterion for separating noisy data. Note that we strictly prohibit splitting $D$ according to $\eta$ after applying DA methods. 

First, we present our analysis based on optimization theory and define the expected risk as follows: 
\begin{align}
R(h) = {\mathbb E}_{(x,y)\sim P}[\ell(h,(x,y))]. \nonumber 
\end{align}
For optimization theory, we are typically interested in the following learning objective: 
\begin{align}
{\cal L}(\theta)={\mathbb E}_{(x,y)\sim P}[\ell(h_{\theta},(x,y))], \nonumber 
\end{align}
where $\theta$ denotes the associated parameters for $h$. Hereafter, we omit the subscript $\theta$ in $h_{\theta}$, i.e., we treat $h$ as parameterized. Then, under different conditions of the training sets, we have 
\begin{align}
{\cal L}^c(\theta)&={\mathbb E}_{(x,y)}[\ell(h,(x,y))] ={\mathbb E}_{(x,y) \sim P^c}[\ell(h,(x,y))] \nonumber \\
{\cal L}(\theta)&={\mathbb E}_{(x,\tilde y)}[\ell(h,(x, \tilde y))] ={\mathbb E}_{(x,y) \sim P}[\ell(h,(x,y))] \nonumber \\
{\cal L}_{DA}(\theta)&={\mathbb E}_{(x,\tilde y)}{\mathbb E}_{x' \sim T(\cdot|x)}[\ell(h, (x', \tilde y))] ={\mathbb E}_{(x,y) \sim P_{DA}}[\ell(h,(x,y))] \nonumber \\
{\cal L}’_{DA}(\theta)&={\mathbb E}_{(x,\tilde y)}{\mathbb E}_{(x',\tilde y') \sim T'(\cdot|x,\tilde y)}[\ell(h, (x', \tilde y'))] ={\mathbb E}_{(x,y) \sim P'_{DA}}[\ell(h,(x,y))] \nonumber 
\end{align}
Based on these learning objectives, we consider the classical optimization process, i.e., Stochastic Gradient Descent (SGD), and introduce the following assumption for illustration. 
\begin{assumption}[Gradient Descent Method] 
\label{assu1}
Assume that the learning objective function ${\cal L}(\theta )$ is smooth, and that the parameters $\theta$ are optimized using the gradient descent method, which is defined as 
\begin{align}
{\theta _{k + 1}} = {\theta _k} - {\alpha _k}{\nabla {\cal L}}({\theta _k}), \nonumber
\end{align} 
where $k$ denotes the epoch index and $\alpha_k$ is the learning rate. 
\end{assumption}
\noindent\textbf{Remarks}: (a) As we consider that $\ell$ is $\sigma$-admissible, the smoothness requirement in this assumption can be easily satisfied. (b) Although we denote the subscript $k$ in $\alpha_k$, unless learning rate decay methods are considered, $\alpha_k$ remains unchanged. 

Why do we consider SGD as the basic optimization method? The reason is that most optimization algorithms used in modern machine learning and deep learning remain gradient-based, including SGD with momentum, Adam, AdamW, and their variants. Therefore, under this assumption, our analysis covers the vast majority of optimization scenarios encountered in practice. As a result, it cannot be considered a strong assumption unless one adopts optimization methods that are independent of loss gradients. 

Analyzing optimization convergence rates is a major topic in optimization theory. Thus, based on the assumptions of SGD, we present the following theorem on the convergence of the learning objective ${\cal L}(\theta)$ using dataset $D$, together with the noise tolerance rate of SGD optimization. 
\begin{theorem}[Noise Tolerance under SGD Optimization]
\label{theo1}
For any $h\in {\cal H}$, assume that the cost function $c$ is convex and satisfies Definition~\ref{defi1}, and let $\theta^{*c} \in argmin_{h_{\theta} \in {\cal H}} {\cal L}^c(\theta)$. Under the SGD optimization scheme in Assumption~\ref{assu1}, training on the dataset $D$ with noise rate $\eta$ yields an averaged iterate $\bar\theta_K$ satisfying 
\begin{align}
{\mathbb E}[{\cal L}(\bar\theta_K)-{\cal L}^c(\theta^{*c})] \le{\cal O}\left(\frac{1}{\sqrt{K}} +\eta\sigma\Delta_{\|\tilde y-y\|}\right), \nonumber 
\end{align}
which implies a noise tolerance characterized by the optimization error $\varepsilon > 0$,
\begin{align}
\label{theo1eq1}
\eta < \frac{\varepsilon}{\sigma\Delta_{\|\tilde y-y\|}}. 
\end{align}
Here, $\Delta_{\| \tilde y- y \|}={\mathbb E}_{y}{\mathbb E}_{\tilde y}\|\tilde y- y\|$. 
\end{theorem}
\noindent\textbf{Remarks}: (a) The convergence rate is sublinear. $\eta\sigma\Delta$ characterizes the optimization bias but does not affect the overall convergence rate. (b) The loss functions are assumed to be convex rather than strongly convex. (c) $\Delta_{\|\tilde y - y\|}$ denotes the label bias induced by noisy labels. (d) Inequality~\eqref{theo1eq1} can, in fact, be written as $\eta \Delta < \varepsilon / \sigma$. (e) The proofs of this and the following theoretical results are provided in Section~\ref{appproofs}. 

Note that here we consider the optimization objective $|{\cal L}(\theta) - {\cal L}^c(\theta^{*c})|$ instead of $|{\cal L}(\theta) - {\cal L}(\theta^*)|$. This choice allows us to explicitly characterize the impact of noisy labels on the optimization process. In contrast, optimizing $|{\cal L}(\theta) - {\cal L}(\theta^*)|$ leads to the classical optimization setting, where the analysis is independent of the presence of noisy labels and therefore cannot capture their effect. 

The above theorem yields several interesting results. For convex optimization problems based on $D^c$, it is well known that the convergence rate is sublinear. The theorem shows that this convergence rate remains unchanged even in the presence of label noise. In particular, the convergence rate is still sublinear, and the only difference is the appearance of an optimization bias term induced by the noise rate $\eta$ and the noise bias $\Delta_{\|\tilde y-y\|}$. When $\eta=0$, which implies $\Delta_{\|\tilde y-y\|}=0$, the result reduces to the classical convergence guarantee. As remarked above, this optimization bias does not affect the convergence rate, suggesting that noisy labels do not significantly increase the learning cost. Nevertheless, the resulting solution may deviate from $\theta^{*c}$ by $\eta\Delta$. The most interesting result is that we observe a noise tolerance for SGD methods, i.e., Inequality~\eqref{theo1eq1}. This indicates that the deviation of the resulting solution from $\theta^{*c}$ is upper bounded by $\varepsilon/\sigma$, i.e., $\eta\Delta < \varepsilon/\sigma$, where $\sigma$ is linked to the properties of the loss function $\ell$. If we consider a large $\varepsilon$, the optimization can tolerate a larger $\eta$, and vice versa. 

Now, we clearly acknowledge how $\eta$ affects the optimization process under the training set $D$. Next, we observe how $\eta$ affects the optimization process under the training set $D$ when applying $T$ or $T’$. We first present the following theorem. 
\begin{theorem}[Noise Tolerance under SGD Optimization with DA]
\label{theo2}
Under the setting of Theorem~\ref{theo1} and the $W_1$ distance defined in Definition~\ref{defi3}, training on the dataset $D$ with the DA kernel $T(\cdot | x)$ under noise rate $\eta$ produces an averaged iterate $\bar\theta_K$ satisfying 
\begin{align}
{\mathbb E}[{\cal L}_{DA}(\bar\theta_{K})-{\cal L}^c(\theta^{*c})]\le {\cal O}\left(\frac{1}{\sqrt{K}}+ \sigma\Delta_{\|x'-x\|}+\eta\sigma\Delta_{\|\tilde y-y\|}\right), \nonumber 
\end{align}
which implies a noise tolerance characterized by the optimization error $\varepsilon > 0$,
\begin{align}
\label{theo2eq1}
\eta < \frac{\varepsilon - \sigma\Delta_{\|x'-x\|}}{\sigma\Delta_{\|\tilde y-y\|}}.  
\end{align}
Here, $\Delta_{\|x'-x\|}={\mathbb E}_{x}{\mathbb E}_{x’}\|x’-x\|$. 
\end{theorem}
\noindent\textbf{Remarks}: (a) The convergence rate remains sublinear. The optimization bias is increased by an additional term $\sigma\Delta_{\|x’-x\|}$, yet this does not affect the overall convergence rate. (b) The loss functions are still assumed to be convex. (c) $\Delta_{\|x’ - x\|}$ denotes the data bias induced by applying $T$, i.e., DA methods. (d) Inequality~\eqref{theo2eq1} can similarly be rewritten as $\eta \Delta_{\|\tilde y - y\|} < (\varepsilon - \Delta_{\|x’-x\|}) / \sigma$. (e) If $\Delta_{\|x’ - x\|} = 0$, then Inequality~\eqref{theo2eq1} reduces to Inequality~\eqref{theo1eq1}. 

Similarly, we consider the the optimization objective $|{\cal L}_{DA}(\theta) - {\cal L}^c(\theta^{*c})|$ instead of $|{\cal L}_{DA}(\theta) - {\cal L}_{DA}(\theta^*_{DA})|$. For this consideration, we align the objective in Theorem~\ref{theo1}, i.e., $|{\cal L}(\theta) - {\cal L}^c(\theta^{*c})|$. This allows us to directly characterize the impact of introducing DA methods in optimization, as the convergence target is the same ${\cal L}^c(\theta^{*c})$. Moreover, we still need to analyze the impact of noisy labels on optimization. 

The above theorem still yields several interesting results. For the convergence rate, after applying $T$, it remains sublinear, which indicates that using DA methods cannot accelerate nor slow down convergence. However, it introduces an additional optimization bias term induced by the data bias, i.e., $\Delta_{\|x’-x\|}=0$. Interestingly, we find that when $\eta=0$, we obtain the convergence rate for $D^c$ after applying $T$. This result indicates that, regardless of whether the training data contains label noise, DA methods always introduce an optimization bias, which may cause the resulting solution to deviate from the original optimal one. As remarked above, when $\Delta_{\|x’-x\|}=0$, we obtain the same result as in Theorem~\ref{theo1}. However, for DA methods, this data bias is unavoidable since their basic operation is to introduce transformations that alter the original information in $x$, rather than correcting or preserving it. Although the convergence rate remains unchanged, the noise tolerance is reduced, i.e., $\eta \Delta_{\|\tilde y - y\|} < (\varepsilon - \Delta_{\|x’-x\|}) / \sigma$, provided that $\Delta_{\|x’-x\|} > 0$. 

This data bias is problematic for the optimization because we cannot guarantee an upper bound on $\Delta_{\|x’-x\|}$, implying that the magnitude of this bias term can become arbitrarily large. Consequently, an arbitrarily large data bias may lead to $\eta < 0$, which indicates that SGD can no longer tolerate label noise. As a result, under noisy labels conditions, any estimate of the deviation between the solution of the learning objective and the optimal solution becomes meaningless. In other words, we cannot guarantee a meaningful connection between the model trained with DA methods and the optimal model. Therefore, one condition should be satisfied, i.e., $\varepsilon \ge \Delta_{\|x’-x\|}$. Under this condition, we observe that the noise tolerance decreases. This is interesting, as before using the DA method, the optimization could tolerate a relatively large amount of noise; however, now the optimization can tolerate less noise. Consequently, the deviation between the solution of ${\cal L}_{DA}(\theta)$ and the optimal solution must increase, as $\sigma\Delta_{\|x'-x\|}+\eta\sigma\Delta_{\|\tilde y-y\|}$. Now, if we consider using $T'$, do the above results still hold? We present the following corollary for further analysis. 
\begin{corollary}
\label{coro1}
Under the setting of Theorem~\ref{theo2} and the $KL$ divergence defined in Definition~\ref{defi2}, additionally assume that $\ell$ is upper bounded by $M$. Then, training on the dataset $D$ with the DA kernel $T'(\cdot | (x,\tilde y))$ under noise rate $\eta$ yields a noise tolerance characterized by the optimization error $\varepsilon$ 
\begin{align}
\eta < \frac{ \varepsilon -M\sqrt{2KL(P||P’_{DA})} }{\sigma\Delta_{\|\tilde y-y\|}}. \nonumber 
\end{align} 
\end{corollary}
\noindent\textbf{Remarks}: (a) The above noise tolerance is similar to the one obtained in Theorem~\ref{theo2}. (b) Considering the behavior of the risk as $KL \to \infty$, this bound may become meaningless.  (c) If $KL = 0$, then the above tolerance reduces to Inequality~\eqref{theo1eq1}.

The above result is similar to that obtained from Theorem~\ref{theo2}, but $\Delta_{\|x'-x\|}$ is replaced by $M\sqrt{KL}$, as $T'$ transforms the labels as well. This implies that, when using $T'$, the condition becomes more restrictive, as we need $\varepsilon \ge M\sqrt{KL}$. Then, if $\sqrt{KL} \ge \Delta_{\|x’-x\|}$, we argue that the noise tolerance becomes smaller when using $T'$. In contrast, if $\sqrt{KL} < \Delta_{\|x’-x\|}$, we consider that using $T$ obtains a smaller noise tolerance. However, we do not need to compare the specific relationship between them, as both $T$ and $T'$ decrease the noise tolerance compared to the result when DA methods are not incorporated. 

One may ask whether assuming that the cost function $c$ is $\mu$-strongly convex can alleviate the reduction in noise tolerance. The answer is negative, since $\eta$ becomes coupled with additional terms, rather than depending solely on $\Delta$. This further complicates the analysis of noise tolerance. The detailed derivation is omitted here and left to the reader. 

In conclusion, based on the above theoretical results, we find that DA methods always  introduce the optimization bias, which agin with the analysis in ~\citep{b11}. However, one normally introduce DA methods aim to increase the robust for the trained models. Then, we give our sight on the learning consistency theory analysis.

We consider the empirical risk minimization (ERM) objective defined as $\hat R(h)=\frac{1}{n}\sum\nolimits_{i=1}^{n}{\ell(h(x_i),y_i)}$. Correspondingly, under different data conditions, we denote the empirical risks by $\hat R^c(h)$, $\hat R(h)$, $\hat R_{DA}(h)$, and $\hat R’_{DA}(h)$, i.e., 
\begin{align}
R^c(h) &= \frac{1}{n}\sum\nolimits_{i=1}^{n}{\ell(h, (x_i,y_i))} \nonumber \\
\hat R(h) &= \frac{1}{n}\sum\nolimits_{i=1}^{n}{\ell(h, (x_i,\tilde y_i))} \nonumber \\
\hat R_{DA}(h) &= \frac{1}{n}\sum\nolimits_{i=1}^{n}{\ell(h, (x’_i,\tilde y_i))} \quad x’_i \sim T(\cdot \mid x_i) \nonumber \\
\hat R’_{DA}(h) &= \frac{1}{n}\sum\nolimits_{i=1}^{n}{\ell(h, (x’_i,\tilde y’_i))} \quad (x’_i,\tilde y’_i) \sim T’(\cdot \mid x_i, \tilde y’_i). \nonumber
\end{align}
The corresponding expected risks are denoted by $R^c(h)$, $R(h)$, $R_{DA}(h)$, and $R'_{DA}(h)$, and their specific forms (see related ${\cal L}(\theta)$) are defined at the beginning of this subsection. Note that we do not consider using $T$ or $T’$ to increase the number of samples, since, under the online DA setting, regardless of how many times $T$ or $T’$ is applied, the generated samples are still treated as a single sample in ERM. If we instead consider using $T$ or $T’$ to increase the sample size, the resulting samples would generally violate the i.i.d. assumption.  

Now, given the empirical risks and the expected risks, according to classical learning theory, we typically need to analyze the learning consistency, i.e., to determine the gap between the empirical risks and the expected risks. We therefore present the following result. 
\begin{theorem}[Learning Consistency]
\label{theo3}
Given the training dataset $D$ with $\eta$-noisy labels and the corresponding definitions of the risk functions, assume that $\ell \in [0, M]$. Then, for any $h \in {\cal H}$ and any $\delta \in (0,1)$, the following inequality holds with probability at least $1-\delta$, 
\begin{align}
\hat R(h)-R(h) \le 2\eta M\sqrt{\frac{\ln(2/\delta)}{2n_{\eta}}} + 2(1-\eta) M \sqrt{\frac{\ln(2/\delta)}{2n_{c}}}. \nonumber 
\end{align}
Here, $n=n_{\eta}+n_{c}$ denotes the total number of data samples in $D$. 
\end{theorem} 
\noindent\textbf{Remarks}: (a) The main order of this upper bound is ${\cal O}(\min\{1/\sqrt{n_{\eta}} , 1/\sqrt{n_c}\})$. (b) If $\eta = 0$, this bound reduces to the well-known learning bound in the literature, i.e., $\varepsilon \le {\cal O}(1/\sqrt{n})$. 

Note that we still assume Assumption~\ref{assu0} holds for $D$. Thus, the above learning bound can be regarded as a combination of the learning bounds for data with noise and data without noise. This bound, in fact, tells us that if $\eta$ is small, the learning gap between $\hat R(h)$ and $R(h)$ tends to approach the clean data side. In contrast, if $\eta$ is large, the learning gap clings to the noisy data side. This result aligns with the above results, which indicates that in real-world applications, it is acceptable to have some label noise in data when training models, unless the noise rate is large enough to affect learning consistency. Interestingly, based on the above bound, we derive the following inequality for estimating $\eta$. 
\begin{corollary}[Noise Bound]
\label{coro2}
Based on Theorem~\ref{theo3}, for any $\varepsilon > 0$, we derive a noise bound:  
\begin{align}
\label{coro2eq1}
\frac{\varepsilon - B}{A-B}\le \eta \le \frac{B-\varepsilon}{B-A}, 
\end{align}
where $A=\sqrt{{(2M)^2\ln(1/\delta)}/{2n_{\eta}}}$ and $B=\sqrt{{(2M)^2\ln(1/\delta)}/{2n_{c}}}$.    
\end{corollary} 
\noindent\textbf{Remarks}: (a) If $A \gg B$, i.e., $n_c \to \infty$, then $\eta \approx \varepsilon / A$. (b) if $A \ll B$, i.e., $n_{\eta} \to \infty$, then $\eta \approx (B - \varepsilon) / B$, i.e., $\eta \approx 1$. 

The above estimation for $\eta$ seems to be meaningless, as in Assumption~\ref{assu0}, we assume that the data can be clearly partitioned into clean and noisy parts. This means that $\eta = n_{\eta}/n$, and when $n_{\eta} \to n$, we have $\eta = 1$. In fact, from the perspective of ~\citep{b12}, the above result can also be regarded as a measure of noise tolerance, i.e., its role is similar to that of Inequality~\eqref{theo1eq1}. Then, for a small $\varepsilon$, we probably need $\eta$ to be confined within the above interval. Here, we explain this result from another perspective. This perspective implies that the effect of noisy labels on learning consistency is both sufficient and necessary with respect to the noise rate. In simple words, given a noise rate, we can estimate the magnitude of the learning gap. Conversely, given an acceptable learning gap, we can estimate the corresponding noise rate. 

Then, we need to verify whether the relationship between the effect of noisy labels on learning consistency and the noise rate still holds after applying $T$ and $T'$ to $D$. We therefore present the following theorem. 
\begin{theorem}[Learning Consistency with DA]
\label{theo4}
Given the training dataset $D$ with $\eta$-noisy labels, the DA kernel $T(\cdot \mid x)$, the corresponding risk function definitions, and the $W_1$ distance defined in Definition~\ref{defi3}, assume that $\ell \in [0, M]$ and that the cost function $c$ is convex and satisfies Definition~\ref{defi1}. Then, for any $h \in {\cal H}$ and any $\delta \in (0,1)$, the following inequality holds with probability at least $1-\delta$, 
\begin{align}
\hat R_{DA}(h)-R(h)\le & \sigma\Delta_{\|x’-x\|}+4M\sqrt{\frac{\ln(3/\delta)}{2n}} \nonumber \\
& +2\eta M \sqrt{\frac{\ln(3/\delta)}{2n_{\eta}}} + 2(1-\eta) M \sqrt{\frac{\ln(3/\delta)}{2n_{c}}}. \nonumber 
\end{align}
Here, $n=n_{\eta}+n_{c}$ represents the number of data samples in $D$ and $\Delta_{\|x'-x\|}={\mathbb E}_{x}{\mathbb E}_{x’}\|x’-x\|$. 
\end{theorem}
\noindent\textbf{Remarks}: (a) The main order of this upper bound is ${\cal O}(1/\sqrt{n} + \min\{1/\sqrt{n_{\eta}} , 1/\sqrt{n_c}\})$. (b) There is a bias term $\Delta_{\|x’-x\|}$, which may lead to a loose or even vacuous bound, even when $n \to \infty$. (c) If $\eta = 0$, the bias term still exists, and therefore continues to affect the tightness of the bound. 

From the above learning bound, we observe that there is an additional bias term on the right-hand side (RHS) compared with the bound in Theorem~\ref{theo3}. This bias term also appears in Inequality~\ref{theo2eq1}. The result indicates that the data bias induced by DA methods affects not only the optimization process but also the learning consistency. As remarked above, when this bias term is large, the learning bound may become loose or even vacuous. Moreover, we observe that the main order of the bound changes, namely, the convergence rate deteriorates, since ${\cal O}(1/\sqrt{n}) \le {\cal O}(1/\sqrt{n} + \min\{1/\sqrt{n_{\eta}}, 1/\sqrt{n_c}\})$. This implies that although the optimization convergence rate remains unchanged after applying DA methods, the convergence rate of the learning bound deteriorates. This observation is interesting because it suggests that even when the bias term is sufficiently small, the above bound is looser than that in Theorem~\ref{theo3}, resulting in a more unstable learning gap, i.e., $\varepsilon \le \varepsilon_{DA}$. In other words, the use of DA methods may render learning consistency uncertain. Nevertheless, the learning gap can still be estimated based on the noise rate. We then state the following result for the application of $T’$. 
\begin{corollary}
\label{coro3}
Given the training dataset $D$ with $\eta$-noisy labels, the DA kernel $T'(\cdot | (x,\tilde y))$, the corresponding risk function definitions, and the KL divergence defined in Definition~\ref{defi2}, assume that $\ell \in [0, M]$ and that the cost function $c$ is convex and satisfies Definition~\ref{defi1}. Then, for any $h \in {\cal H}$ and any $\delta \in (0,1)$, the following inequality holds with probability at least $1-\delta$,
\begin{align}
\hat R'_{DA}(h)-R(h)\le & M\sqrt{2KL(P||P'_{DA})}+4M\sqrt{\frac{\ln(3/\delta)}{2n}} \nonumber \\
&+2\eta M \sqrt{\frac{\ln(3/\delta)}{2n_{\eta}}} + 2(1-\eta) M \sqrt{\frac{\ln(3/\delta)}{2n_{c}}}. \nonumber 
\end{align}
Here, $n=n_{\eta}+n_{c}$ represents the number of data samples in $D$. 
\end{corollary} 
\noindent\textbf{Remarks}: (a) The main order of this upper bound is still ${\cal O}(1/\sqrt{n} + \min\{1/\sqrt{n_{\eta}} , 1/\sqrt{n_c}\})$. (b) There is a bias term $\sqrt{KL}$, which may lead to a loose or even vacuous bound, even when $n \to \infty$. (c) If $\eta = 0$, the bias term still exists, and therefore continues to affect the tightness of the bound. 

Similarly, we have $\varepsilon \le \varepsilon'_{DA}$, and the learning gap can still be estimated. The only difference, compared with the bound in Theorem~\ref{theo4}, is the $\sqrt{KL}$ term. Therefore, we omit further discussion of this bound. Based on the above results, we present the following bound for estimating $\eta$.
\begin{corollary}[Noise Bound with DA]
\label{coro4}
Based on Theorem~\ref{theo4} or Corollary~\ref{coro3}, for any $\varepsilon > 0$, we derive a noise bound: 
\begin{align}
\label{coro4eq1}
\frac{\varepsilon - (B+C+\Delta)}{A-B}\le \eta \le \frac{(B+C+\Delta)-\varepsilon}{B-A}, 
\end{align}
where $A=\sqrt{{(2M)^2\ln(1/\delta)}/{2n_{\eta}}}$, $B=\sqrt{{(2M)^2\ln(1/\delta)}/{2n_{c}}}$,  and $C=\sqrt{{(4M)^2\ln(1/\delta)}/{2n}}$. $\Delta = \sigma\Delta_{\|x’-x\|} $ or $\Delta = M\sqrt{2KL(P||P’_{DA})}$. 
\end{corollary}
\noindent\textbf{Remarks}: (a) If $A \gg B$, i.e., $n_c \to \infty, n \to \infty$, then $\eta \approx (\varepsilon - \Delta) / A$. (b) if $A \ll B$, i.e., $n_{\eta} \to \infty, n \to \infty$, then $\eta \approx (B+ \Delta - \varepsilon) / B$.  

Here, we observe that the sufficient and necessary relationship between the effect of noisy labels on learning consistency and the noise rate no longer holds. This breakdown is induced by $\Delta$. We then provide the following figure to illustrate the results in~\eqref{coro2eq1} and \eqref{coro4eq1}, where $\eta_{DA}$ denotes the quantity defined in~\eqref{coro4eq1}.  First, we assume that $\varepsilon \ge (B+C+\Delta)$; then, we can still estimate the interval for the noise rate as follows. 
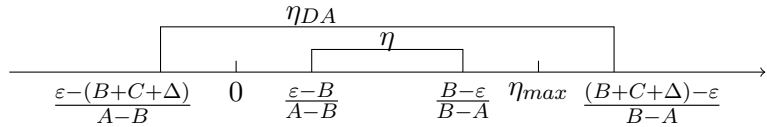
\begin{figure}[h]
\centering
\begin{tikzpicture}
\draw[->] (-5,0) -- (5,0) node[right] {$$};

\draw (-1,0) -- (-1,0.3) -- (1,0.3) -- (1,0);

\node at (0,0.4) {$\eta$};

\draw (-3,0) -- (-3,0.6) -- (3,0.6) -- (3,0);

\node at (-1.0, 0.75) {$\eta_{DA}$};

\draw (-2,0) -- (-2,0.15);
\draw (2, 0) -- (2,0.15);

\node[below] at (-3.5,0) {$\frac{\varepsilon - (B+C+\Delta)}{A-B}$};
\node[below] at (3.5,0) {$\frac{(B+C+\Delta)-\varepsilon}{B-A}$};

\node[below] at (-1,0) {$\frac{\varepsilon - B}{A-B}$};
\node[below] at (1,0) {$\frac{B-\varepsilon}{B-A}$};

\node[below] at (-2,0) {$0$};
\node[below] at (2,0) {$\eta_{max}$};

\end{tikzpicture}
\end{figure}

This indicates that DA methods also expand the estimation range of $\eta$, i.e., the bound for noise becomes looser. Under this assumption, the sufficient and necessary relationship still holds. However, when this assumption does not hold, we have the following result, where the estimation of $\eta$ is totally meaningless. Correspondingly, the sufficient and necessary relationship breaks. 

\begin{figure}[h]
\centering
\begin{tikzpicture}

\draw[->] (-5,0) -- (5,0) node[right] {$$};

\draw (1,0) -- (1,0.4) -- (3,0.4) -- (3,0);
\node at (2,0.5) {$\eta$};

\draw (-4,0) -- (-4,0.4) -- (-1,0.4) -- (-1,0);
\node at (-2.5, 0.5) {$\eta_{DA}$};

\draw (0,0) -- (0,0.15);
\draw (4, 0) -- (4,0.15);

\node[below] at (-4,0) {$\frac{\varepsilon - (B+C+\Delta)}{A-B}$};
\node[below] at (-1.5,0) {$\frac{(B+C+\Delta)-\varepsilon}{B-A}$};

\node[below] at (1,0) {$\frac{\varepsilon - B}{A-B}$};
\node[below] at (3,0) {$\frac{B-\varepsilon}{B-A}$};

\node[below] at (0,0) {$0$};
\node[below] at (4,0) {$\eta_{max}$};

\end{tikzpicture}
\end{figure}

In conclusion, applying DA methods not only decreases the noise tolerance for SGD optimization but also makes the learning consistency uncertain. Notes that the above learning bounds are Hoeffding-type bounds. If one employs Bernstein-type inequalities~\citep{b3} to obtain tighter bounds, or uses VC-dimension-based bounds~\citep{b4} to analyze the classification case, similar results can be derived; the proof is left to the reader. 

The above analysis results finally provide the insight that the noise rate $\eta$ may differ from the noise rate after applying DA methods, i.e., $\eta \ne \eta_{DA}$. A natural question is whether $\eta \le \eta_{DA}$ or $\eta > \eta_{DA}$, and under what conditions each case holds. In the following, we investigate these questions. 

In the beginning of this subsection, we introduce a strong assumption, i.e., Assumption~\ref{assu0}. Based on this assumption, the analysis appears to be limited to symmetric label noise, i.e., $P(\tilde y \mid y)$, thereby excluding scenarios with semantic noise, i.e., instance-dependent noise $P(\tilde y \mid x, y)$. In the subsequent analysis, we no longer rely on this assumption. Here, we first explain why we no longer rely on this assumption. The analysis derived under this assumption is applicable only under limited conditions, which restricts its applicability to more general cases. As a result, the corresponding conclusions provide limited insight into real-world learning problems. Nevertheless, even under such a strong assumption, we can still conduct a further investigation of the above question. Interested readers may find it worthwhile to explore this direction over a cup of coffee. 

\subsection{Analysis}
\label{sec3.3}

In this subsection, we present the formal analysis without relying on the strong assumption introduced in the previous subsection. Instead, we consider several standard settings that are commonly satisfied in real-world learning scenarios.

We firstly define the following risk functions: the relative expected risk for any $h_1, h_2 \in {\cal H}$,
\begin{align}
\gamma(h_1,h_2)={\mathbb E}_{(x,y)}[\ell(h_1, (x,y))-\ell(h_2, (x,y))] , \nonumber 
\end{align}
and the relative empirical risk for any $h_1, h_2 \in {\cal H}$, 
\begin{align}
\hat \gamma(h_1,h_2)=\frac{1}{n}\sum\limits_{i=1}^{n}[\ell(h_1, (x_i,y_i))-\ell(h_2, (x_i,y_i))]. \nonumber 
\end{align}
Similarly, under different data conditions, we define $\gamma^c$, $\gamma_{DA}$, and $\gamma'_{DA}$, consistent with the previous settings. Note that we do not introduce the absolute value operation in the substitution for the two losses, so the resulting relative risks may take negative values and must be accounted for accordingly. We then introduce a measure of the discrepancy between any two hypotheses $h, h' \in {\cal H}$ as follows. 

\begin{definition}
\label{defi3}
Let $\cal H$ be a hypothesis space consisting of measurable functions $h:{\cal X} \to {\cal Y}$, where $\cal X$ and $\cal Y$ are measurable spaces, and $\mu$ be a probability measure on $\cal X$. 
Assume $h, h' \in {\cal H}$ are $\mu$-integrable. Define a distance between any two hypotheses $h, h' \in {\cal H}$ by the ${\mathbb L}_1(\mu)$-norm:
\begin{align} 
d_{\mu}(h,h') := \left\| h - h' \right\|_{{\mathbb L}_1(\mu)} = {\mathbb E}_{x \sim \mu}[ \left | h(x) - h'(x) \right| ].  \nonumber 
\end{align} 
We consider that this distance is uniformly bounded, i.e.,
\begin{align}
0 \le d_{\mu}(h,h’) \le \beta, \quad \forall h, h' \in {\cal H}, \nonumber 
\end{align}
for some finite constant $\beta > 0$. 
\end{definition}
\noindent\textbf{Remarks}: (a) By definition of the distance, $d_{\mu}(h,h’) = 0$ if and only if $h(x) = h'(x)$ $\mu$-almost everywhere on $\cal X$. (b) Note that we consider the distance is upper bounded only to exclude the case of infinity. (c) Hereafter, we denote $d_x$ by $d_{\mu}$, as this notation preserves the dependence on the underlying measure $\mu$ without ambiguity.

We define $h^*\in argmin_{h\in{\cal H}}{R(h)}$ and $h^{*c}\in argmin_{h\in{\cal H}}{R^c(h)}$. $h^*_{DA}\in argmin_{h\in{\cal H}}{R_{DA}(h)}$ and $h^{*c}_{DA}\in argmin_{h\in{\cal H}}{{\mathbb E}_{(x,y)}{\mathbb E}_{x' \sim T(\cdot|x)}[\ell(h, (x’,y))]}$, or, $h^*_{DA}\in argmin_{h\in{\cal H}}{R’_{DA}(h)}$ and $h^{*c}_{DA}\in argmin_{h\in{\cal H}}{{\mathbb E}_{(x,y)}{\mathbb E}_{(x’,y’) \sim T(\cdot|x,y)}[\ell(h, (x’,y’))]}$. Here, we treat $h^{*c}_{DA}$ as a reference to $h^{*c}$. Here, we assume that applying the same $T$ or $T’$ to $D^c$ yields an optimal hypothesis that differs from $h^{*c}$, since $T$ or $T’$ is not $1\text{-}1$. Now, based on the definition of optimal hypotheses, we introduce the following notations:
\begin{align}
\varepsilon_{\eta}&:=d_{x}(h^*,h^{*c}), \nonumber \\
\varepsilon_{\eta_{DA}} &:= d_{x’}(h^*_{DA},h^{*c}_{DA}) \nonumber 
\end{align}
These notations indicate that the noise errors are induced by noisy labels and can represent all kinds of noisy conditions. Note that these noise errors are proportional to the noise rates rather than equal to them, i.e., $\varepsilon_{\eta} \propto \eta$. Consequently, larger noise errors correspond to higher noise rates, and vice versa. Here, the validity of the above definition is independent of whether $x \sim P$ or $x \sim P^c$, and whether $x’ \sim P_{DA}$ or $x’ \sim P^{c}_{DA}$. 

Now, having introduced all the necessary notations, we present the following theorem. 
\begin{theorem}
\label{theo4p5}
Given any training set, the definition of the relative risks, and assuming $\ell \in [0,M]$, for any $h_1, h_2 \in {\cal H}$, the following inequality holds with probability at least $1-\delta$: 
\begin{align}
\hat\gamma(h_1,h_2)-\gamma(h_1,h_2)\le 4M \sqrt{\frac{\ln(1/\delta)}{2n}}.  \nonumber 
\end{align}
\end{theorem}
\noindent\textbf{Remarks}: (a) The main order of this bound is ${\cal O}(1/\sqrt{n})$. (b) Here, $h_1$ and $h_2$ can be any hypotheses in ${\cal H}$. For the non-trivial case, we let $h_1$ and $h_2$ be different. If $h_1$ is equal to $h_2$, this leads to a trivial case, as the RHS is always non-negative. 

The above bound is also a Hoeffding-type bound, and its tightness is guaranteed as $n \to \infty$. This indicates that the relationship between empirical and expected risks can still be analyzed via high-probability bounds. Based on this observation, we then consider replacing $h_1$ and $h_2$ with specific hypotheses and derive the following corollary. 
\begin{corollary}
\label{coro4p5}
Based on Theorem~\ref{theo4p5}, considering the training set $D$ with $\eta$-noisy labels and the DA kernel $T$, for any $h \in {\cal H}$, we have the following inequalities: 
\begin{align}
\hat\gamma(h,h^*) -\gamma(h,h^{*c}) \le \sigma \varepsilon_{\eta}+ 4M\sqrt{\frac{\ln(1/\delta)}{2n}}, \nonumber 
\end{align}
and 
\begin{align}
\hat\gamma_{DA}(h,h^*_{DA}) -\gamma_{DA}(h,h^{*c}_{DA}) \le \sigma \varepsilon_{\eta_{DA}}+ 4M\sqrt{\frac{\ln(1/\delta)}{2n}}. \nonumber 
\end{align}
Here, $h^{*c}$, $h^{*c}_{DA}$, $\varepsilon_{\eta}$, and $\varepsilon_{\eta_{DA}}$ are defined above. When considering $T’$, $\hat\gamma’_{DA}(h,h^*_{DA}) -\gamma’_{DA}(h,h^{*c}_{DA}) $ yields the same result as that obtained for $T$. 
\end{corollary}
\noindent\textbf{Remarks}: (a) The main orders of the above bounds are ${\cal O}(1/\sqrt{n})$. (b) There are bias terms on the RHS of the above bounds, namely $\varepsilon_{\eta}$ and $\varepsilon_{\eta_{DA}}$, which may lead to loose or vacuous bounds even as $n \to \infty$. 

The above corollary provides an insight suggesting that, for any learned $h$ under $D$, with or without DA methods, the relative risk gaps are bounded by the noise errors. Here, we consider the gap between $\hat \gamma(h, h^*)$ and $\gamma(h, h^{*c})$ in order to examine the effect of noisy labels. This setting is consistent with Theorem~\ref{theo3}, and is derived under milder assumptions. Through this corollary, we also observe that the only difference between these two bounds lies in $\varepsilon_{\eta}$ and $\varepsilon_{\eta_{DA}}$. From the previous analysis, we have shown that under a strong assumption, they are different. However, once this assumption is no longer imposed, it is unclear whether such a discrepancy still persists. Once this discrepancy is identified and the relative magnitudes of $\varepsilon_{\eta}$ and $\varepsilon_{\eta_{DA}}$ are determined, we can explicitly characterize the tightness of the above two bounds, identify which bound is tighter and which is looser, and consequently draw conclusions regarding the effect of DA methods on the relative risk gap. In simple terms, when $\varepsilon_{\eta_{DA}} \ge \varepsilon_{\eta}$, the risk gap increases, implying that DA is not beneficial for learning, and vice versa. Based on the classical triangle inequality, we then establish the following result. 
\begin{theorem}
\label{theo6}
By the definitions of $\varepsilon_{\eta}$ and $\varepsilon_{\eta_{DA}}$, the inequality follows: 
\begin{align}
\label{theo6eq1}
\left| \varepsilon_{\eta} - \varepsilon_{\eta_{DA}}\right|\le d_{x,x’}(h^*,h^*_{DA})+d_{x,x'}(h^{*c},h^{*c}_{DA}) 
\end{align}
Here, $d_{x,x’}(h^*,h^*_{DA}):={\mathbb E}_{x,x’}[|h^*(x)-h^*_{DA}(x')|]$ and $d_{x,x'}(h^{*c},h^{*c}_{DA}):={\mathbb E}_{x,x’}[|h^{*c}(x)-h^{*c}_{DA}(x')|]$
\end{theorem}
\noindent\textbf{Remarks}: (a) $\varepsilon_{\eta}$ and $\varepsilon_{\eta_{DA}}$ must be different unless $d_{x,x'}(h^*,h^*_{DA}) \to 0$ and $d_{x,x'}(h^{*c},h^{*c}_{DA}) \to 0$. However, this special case cannot occur because neither $T$ nor $T'$ is $1\text{-}1$. (b) A trivial case is $\varepsilon_{\eta} = \varepsilon_{\eta_{DA}} = 0$. 

The above result establishes that, provided $T$ and $T'$ are not $1\text{-}1$ transformations, $\varepsilon_{\eta}$ and $\varepsilon_{\eta_{DA}}$ are generally different. However, it does not characterize their ordering. To determine whether $\varepsilon_{\eta} \ge \varepsilon_{\eta_{DA}}$ or $\varepsilon_{\eta_{DA}} \ge \varepsilon_{\eta}$, we introduce the following conditions. First, a sufficient condition under which $\varepsilon_{\eta_{DA}} \ge \varepsilon_{\eta}$ holds is given as follows. 

\begin{theorem}[Trivial Sufficient Condition]
\label{theo7}
Let $\epsilon = {\mathbb E}_{x’,\tilde y}[\ell(h^*_{DA}, (x’,\tilde y))]$ be a sufficiently small and $d_{x}(h^*,h^{*c}) \le \eta_{max}$. Then, we have $\varepsilon_{\eta_{DA}} \ge \varepsilon_{\eta}$ if the following condition holds: 
\begin{align}
\label{theo7eq1}
\max\left\{ \alpha \Delta_{\|\tilde y-y\|_{DA}}, \frac{\Delta_{\|x'-x\|}}{r(x)}, \frac{1}{2} \left( \alpha \Delta_{\|\tilde y-y\|_{DA}} + \frac{\Delta_{\|x'-x\|}}{r(x)} \right) \right\} \ge \frac{\sigma\eta_{max}+\epsilon}{M} 
\end{align}
Here, $\alpha>0$ is a constant, $\tilde y$ can be replaced by $\tilde y’$, if considering $T’$. $r(x)$ denotes a margin condition such that $h(x’) \ne h(x) = y$ whenever $|x’-x|>r(x)$ for any $h$, which characterizes the failure condition of prediction. 
\end{theorem}
\noindent\textbf{Remarks}: (a) $\Delta_{\|\tilde y-y\|_{DA}}$ represents the label bias in $D$ after applying $T$. (b) $\Delta_{\|x’-x\|} / r(x)$ represents the ratio of data bias in $D$ after applying $T$ to $r(x)$. (c) More rigorously, $d_{x}(h^*,h^{*c}) \le \alpha\eta_{max}$, as $d_{x}(h^*,h^{*c}) \propto P[h^*(x)\ne h^{*c}(x)]$. (d) $\epsilon$ is a small constant due to $h^*_{DA} \in argmin R_{DA}(h)$. 

Here, this condition is regarded as trivial, as it follows immediately from $\varepsilon_{\eta_{DA}} \ge \eta_{max} \ge \varepsilon_{\eta}$. This implies a sufficient condition under which the event $\varepsilon_{\eta_{DA}} \ge \eta_{max}$ occurs. Although this condition is rather trivial, it leads to several interesting observations. 
(a) If $\Delta_{\|\tilde y-y\|_{DA}} \ge RHS$, it implies that the DA method amplifies the spread of noisy labels, leading to $\varepsilon_{\eta_{DA}} \ge \varepsilon_{\eta}$. For synthetic or semi-synthetic datasets, this quantity can be computed directly. Indeed, in such settings, both $\eta$ and $\eta_{DA}$ are available and can be explicitly estimated. However, in real-world applications, these quantities are generally inaccessible, making this condition difficult to verify in practice. Therefore, we do not further pursue this condition, as well as the mixture condition. 
(b) One condition that remains computable during training, even in real-world applications, is $\Delta_{\|x’-x\|}/r(x)$. Here, $r(x)$ denotes the prediction margin of $h$, which is typically treated as a constant and can be estimated empirically through repeated experiments. If $\Delta_{\|x’-x\|}/r(x) \ge RHS$, the condition directly indicates that the noisy labels are propagated by the DA methods. Moreover, it suggests avoiding the inclusion of $min \|h(x')-h(x)\|$ during the training, as 
\begin{align}
\mathop {\lim }\limits_{r(x) \to 0 } { \frac{\Delta_{\|x’-x\|}}{r(x)}}=\infty. \nonumber 
\end{align}
On the contrary, one could consider the chain $\varepsilon_{\eta} \ge \eta_{max} \ge \varepsilon_{\eta_{DA}}$ to derive a condition. However, this is unnecessary, since in machine learning settings we do not typically use original data whose label noise exceeds the maximum threshold. Moreover, when this condition does not hold, the relative magnitudes of $\varepsilon_{\eta}$ and $\varepsilon_{\eta_{DA}}$ remain indeterminate. Therefore, we instead state the following non-trivial condition. 
\begin{theorem}[Non-Trivial Sufficient Condition]
\label{theo7p5}
Assume that $h(x) \in [0, \upsilon], \forall x \in {\cal X}, \forall h \in {\cal H}$. Then, for any $h \in {\cal H}$, 
\begin{align}
\label{theo7p5eqs}
d_{x}(h,h^*)-d_{x’}(h,h^*_{DA})&\ge{\cal K} \Rightarrow \varepsilon_{\eta} \ge \varepsilon_{\eta_{DA}}, \\
d_{x’}(h,h^*_{DA})-d_{x}(h,h^*)&\ge{\cal K} \Rightarrow \varepsilon_{\eta_{DA}} \ge \varepsilon_{\eta}. 
\end{align}
Here, ${\cal K}={\mathbb E}_{x}h(x)+{\mathbb E}_{x’}h(x')+2\upsilon$. 
\end{theorem}
\noindent\textbf{Remarks}: (a) The boundedness assumption on $h$ is mild in classification settings, whereas it is stronger in regression settings. When extending to regression, it is typically necessary to post-process the outputs of $h$, for example via a sigmoid transformation. Further details will be provided later. (b) Note that $d_{x}(h,h^*)-d_{x’}(h,h^*_{DA})\le{\cal K} \not \Rightarrow d_{x’}(h,h^*_{DA})-d_{x}(h,h^*)\ge{\cal K} $, and $d_{x’}(h,h^*_{DA})-d_{x}(h,h^*)\le{\cal K} \not \Rightarrow d_{x}(h,h^*)-d_{x’}(h,h^*_{DA})\ge{\cal K}$. (c) The above conditions are relatively loose; further discussion is deferred to Section~\ref{sec5}.

Based on the above, it is possible to estimate whether label noise is propagated by DA methods during training or after evaluation. We therefore introduce the following approximation methods. 

Let $h^*(x) = \tilde y$ and $h^*_{DA}(x’) = \tilde y $, or $h^*_{DA}(x’) = \tilde y’ $ when considering $T’$. Then, letting $\hat h \in argmin \hat {\phi}(h)$, we obtain 
\begin{align}
{\mathbb E}_{x,\tilde y}[\ell(\hat h, (x,\tilde y))]-{\mathbb E}_{x',\tilde y}[\ell(\hat h, (x',\tilde y))] &\ge{\cal K}, \nonumber \\
{\mathbb E}_{x’,\tilde y}[\ell(\hat h, (x’,\tilde y))]-{\mathbb E}_{x,\tilde y}[\ell(\hat h, (x,\tilde y))]&\ge{\cal K}, \nonumber 
\end{align}
where ${\cal K}={\mathbb E}_{x}\hat h(x)+{\mathbb E}_{x’}\hat h(x')+2\upsilon$. Here, $\hat{\phi}(h)$ denotes either $\hat R(h)$, $\hat R_{DA}(h)$, or $\hat R’_{DA}(h)$. In practice, the expectation $\mathbb{E}[\cdot]$ can be approximated by the corresponding empirical average. 

Note that both the above conditions and the trivial condition have limitations. When neither condition holds, the relative magnitudes of $\varepsilon_{\eta}$ and $\varepsilon_{\eta_{DA}}$ remain unclear. This issue will be further discussed in Section~\ref{sec5}. 

Finally, we provide the following additional analyses before concluding this section.  
\begin{theorem}
\label{theo5}
Given the training set $D$ with $\eta$-noisy labels, the DA kernel $T$ or $T’$, the corresponding risk function definitions, and the $W_1$ distance and the KL divergence. Assume that $\ell \in [0, M]$ that is $\sigma$-admissible. Then for any $h \in {\cal H}$, the discrepancy between $R(h)$ and $R_{DA}(h)$ satisfies, with probability at least $1-\delta$: 
\begin{align}
R(h)-R_{DA}(h) \le 2\sigma\Delta_{\|\tilde y - y\|} + \sigma\Delta_{\|x’-x\|} + CM\sqrt{\frac{\ln(7/\delta)}{2n}}, \nonumber 
\end{align}
or the discrepancy between $R(h)$ and $R’_{DA}(h)$ satisfies: 
\begin{align}
R(h)-R’_{DA}(h) \le 2\sigma\Delta_{\|\tilde y - y\|} + M\sqrt{2KL(P||P’_{DA})} + CM\sqrt{\frac{\ln(7/\delta)}{2n}}. \nonumber 
\end{align}
Here, $C>0$ is a constant, $\Delta_{\|x'-x\|}={\mathbb E}_{x}{\mathbb E}_{x’}\|x’-x\|$, and $\Delta_{\| \tilde y- y \|}={\mathbb E}_{y}{\mathbb E}_{\tilde y}\|\tilde y- y\|$. 
\end{theorem}
\noindent\textbf{Remarks}: (a) The main order of the above bounds is ${\cal O}(1/\sqrt{n})$. (b) There are two bias terms in each bound, which may harm the tightness of the bounds. 

The above bounds characterize the difference between $R(h)$ and $R_{DA}(h)$ for any $h \in {\cal H}$. We observe that, even when $\Delta_{\|x’-x\|} \to 0$, the bound remains loose in the presence of label bias. Moreover, since $\Delta_{\|\tilde y -y\|}$ is unknown in real-world settings, estimating the gap between $R(h)$ and $R_{DA}(h)$ is challenging. If $\Delta_{\|\tilde y -y\|} = 0$, i.e., there is no label noise, then as $\Delta_{|x’-x|} \to 0$, we have that $R(h)$ and $R_{DA}(h)$ become asymptotically equivalent. Based on this observation, we derive the following corollary.  
\begin{corollary}
\label{coro5}
Under the settings of Theorem~\ref{theo5} and the definitions of $h^*$ and $h^*_{DA}$, the following inequality holds with probability at least $1-2\delta$:
\begin{align}
R(h^*)-R_{DA}(h^*_{DA}) \le \sigma d_{x’,x}(h^*,h^*_{DA}) + 2\left \{ 2\sigma\Delta_{\|\tilde y - y\|} + \sigma\Delta_{\|x'-x\|} + CM\sqrt{\frac{\ln(7/\delta)}{2n}} \right\}, \nonumber 
\end{align}
or 
\begin{align}
R(h^*)-R’_{DA}(h^*_{DA}) \le &\sigma d_{x’,x} + 2\left \{ 2\sigma\Delta_{\|\tilde y - y\|} + M\sqrt{2KL(P||P’_{DA})} + CM\sqrt{\frac{\ln(7/\delta)}{2n}} \right\}. \nonumber 
\end{align}
Here, $d_{x’,x}(h^*,h^*_{DA}):= {\mathbb E}_{x’,x}[|h^*(x’)-h^*_{DA}(x)|]$.
\end{corollary}
\noindent\textbf{Remarks}: (a) The main order of the above bounds is still ${\cal O}(1/\sqrt{n})$. (b) An additional bias term $d_{x',x}$ is introduced in the above bounds. 

The above results indicate the discrepancy between $R(h^*)$ and $R(h^*_{DA})$. Note that these bounds do not indicate which risk is larger; rather, they only quantify the gap between the two risks. These bounds reveal that multiple factors can contribute to enlarging this gap. Through these bounds, we can draw another conclusion: using DA methods further increases the uncertainty of uncertain learning procedures. How can this conclusion be understood? Machine learning procedures are essentially unstable, and introducing DA significantly increases this uncertainty. We argue that although DA may bring improvements in evaluation performance (see the experiments by~\cite{b13}), the risk introduced by DA itself is larger than its benefits.

\section{A Domain Generalization Perspective}
\label{sec4}

\subsection{Setup}
\label{sec4.1}

\textbf{Basic Data Assumption}. We clarify the basic data structure for subsequent analysis. Under the domain generalization (DG) setting, let ${\cal X}$ and ${\cal Y}$ be the input and output spaces, respectively. The clean dataset without labeling noise is denoted as 
\begin{align}
{\cal E}^c:= \{D^c_1, \ldots, D^c_E\} = \left\{\{(x_i,y_i)_1\}_{i=1}^{n_1}, \ldots, \{(x_i,y_i)_E\}_{i=1}^{n_E} \right\}. \nonumber 
\end{align}
Here, each pair $(x^e, y^e) \in {\cal X} \times {\cal Y}$, under any environment $e \in [E]$, is independently and identically drawn from an unknown joint distribution $P(x^e,y^e)$, i.e., $(x^e,y^e) \sim_{i.i.d.} P(x^e,y^e)$, where $P(x^e,y^e)$ is denoted by $P_e$ for simplicity. In the DG setting, we assume that $P_e \ne P_{e’}, \forall e \ne e' \in [E]$. Thus, the data are independently and non-identically distributed (i.n.d.). The dataset ${\cal E}^c$ with size $E\sum_i^E{n_i}$ is used as the training set for learning models, i.e., the \textit{source} set. Correspondingly, a test dataset ${\cal E}_t^c$, the \textit{target} set, is assumed to be drawn from related yet different distributions from ${\cal E}^c$, i.e., $P_e \ne P_{e_t},\forall e \in [E],\forall e_t \in [E_t]$, where $E_t$ denotes the number of test domains or environments in ${\cal E}_t^c$. Moreover, no information about the test set is available during training. 

Typically, we consider the distribution set for ${\cal E}^c$, denoted by ${\bf{P}}^c=\{P^c_1, \ldots, P^c_E\}$, and the distribution set for ${\cal E}^c_t$, denoted by ${\bf{P_t}}^c=\{{P_1}^c, \ldots, {P_{E_t}}^c\}$. Most importantly, the source and target distribution sets satisfy ${\bf{P}}^c \cap {\bf{P_t}}^c = \emptyset$, with both ${\bf{P}}^c \subset {\cal P}^c$ and ${\bf{P_t}}^c \subset {\cal P}^c$. 

Then, according to the above, the noisy-label source dataset is denoted by 
\begin{align}
{\cal E}:= \{D_1, \ldots, D_E\} = \left\{\{(x_i,\tilde y_i)_1\}_{i=1}^{n_1}, \ldots, \{(x_i,\tilde y_i)_E\}_{i=1}^{n_E} \right\}. \nonumber
\end{align}
Similarly, ${\cal E}_t$ denotes the noisy-label target dataset, while ${\bf{P}}$ and ${\bf{P_t}}$ denote the corresponding distribution sets. Most importantly, ${\bf{P}} \cap {\bf{P_t}} = \emptyset$, with both ${\bf{P}} \subset {\cal P}$ and ${\bf{P_t}} \subset {\cal P}$. 

\noindent\textbf{Data Augmentation}. The DA kernel remains the same as that in the previous section, namely $T(\cdot|x)$. Since the analysis for the kernel $T'$ is analogous to that for $T$, with the only difference being that the term $\sigma \Delta$ replaces $M\sqrt{KL}$, we restrict our attention to $T$ in this section. Then, we consider the source distribution sets ${\bf{P}}_{DA}$, satisfying the same properties as the above distribution sets. Note that DA is not applied to the target domains. Therefore, the corresponding distribution set for the target domains remains ${\bf P_t}$. 

\noindent\textbf{Others}. The loss function remains $\ell$, which is assumed to be bounded by $M$ and $\sigma$-admissible. In the previous section, we considered the KL divergence to measure the difference between two distributions under the assumption of absolute continuity. However, in DG settings, this assumption does not necessarily hold. Therefore, we consider alternative measures as follows. 
\begin{definition}[Jensen-Shannon (JS) Divergence]
\label{defi4}
Let $P$ and $Q$ be two probability distribution, and let $M=(P+Q)/2$ be their mixture distribution. The JS divergence is defined as 
\begin{align}
JS(P||Q)=\frac{1}{2}KL(P||M) + \frac{1}{2}KL(Q||M). \nonumber 
\end{align}
Additionally, the JS distance is defined as the square root of the JS divergence: 
\begin{align}
d_{JS}(P,Q)=\sqrt{JS(P||Q)}, \nonumber 
\end{align}
which is a true metric on the space of probability distributions. 
\end{definition} 
\noindent\textbf{Remarks}: (a) $d_{JS}(P,Q)\ge0; d_{JS}(P,Q) = 0 \; iff \; P=Q$. (b) $d_{JS}(P,Q)=d_{JS}(Q,P)$. (c) $d_{JS}(P,Q)\le d_{JS}(P,M) + d_{JS}(M,Q)$ for any distribution $M$. (d) $d_{JS}(P,Q) \le \sqrt{\ln2}$ when the JS divergence is defined using the natural logarithm.

Moreover, we introduce a generalized JS divergence, since DG typically involves multiple domains. 
\begin{definition}[Generalized JS Divergence]
\label{defi5}
Let $\{P_1,\ldots,P_k\}$ be probability distributions (discrete or continuous), and let $\pi=(\pi_1,\cdots ,\pi_k )$ be the set of weights such that $\pi_i > 0$ and $\sum_{i=1}^{k}{\pi}=1$. The generalized JS divergence is defined as
\begin{align}
JSD_{\pi}(P_1,\ldots,P_k) = \sum\limits_{i=1}^{k}{\pi_i KL(P_i||M)}, \nonumber 
\end{align}
where $M$ is the weighted mixture distribution: $M=\sum_{i=1}^{k}{\pi_i P_i}. $
\end{definition}
\noindent\textbf{Remarks}: (a) We typically consider uniform weights $\pi_i = 1/k$, in which case $JSD(P_1,\ldots,P_k) = (1/k)\sum KL(P_i||M)$. (b) Definition~\ref{defi4} corresponds to the special case $k=2$.

\subsection{Analysis}
\label{sec4.2}

We denote the expected risk and empirical risk by ${\cal R}(h) = {\mathbb E}_{P\sim {\bf P}}R_P(h) $ and $\hat {\cal R}(h) = ({1}/{E})\sum_s^E{\hat R_s(h)}$, respectively. We then consider the following relative risk functions: the relative expected risk for any $h_1, h_2 \in {\cal H}$,
\begin{align}
\Gamma(h_1,h_2)={\mathbb E}_{P\sim {\bf{P}}}[\gamma_P(h_1, h_2)] =\frac{1}{E}\sum\limits_{s=1}^E{\gamma_{P_s}(h_1, h_2)} , \nonumber 
\end{align}
and the relative empirical risk for any $h_1, h_2 \in {\cal H}$, 
\begin{align}
\hat \Gamma(h_1,h_2)=\frac{1}{E\bar n}\sum\limits_{s=1}^{E}{\hat \gamma_s(h_1,h_2)}.  \nonumber 
\end{align}
Similarly, under different data conditions, we define $\gamma^c$ and $\gamma_{DA}$, consistent with the previous settings. Here, ${\mathbb E}_{P \sim {\bf{P}}} [\cdot]= (1/E)\sum_{s=1}^E[\cdot]$, since $P_i \ne P_j, \forall i \ne j \in [E]$. Note that defining $\hat \gamma(h_1,h_2) = (1/\bar n) \sum[\cdot]$ is convenient for the subsequent analysis without loss of generality. If this is regarded as a strong assumption, one may instead consider $\hat \Gamma(h_1,h_2) = (1/E)\sum_{s=1}^E { \hat \gamma_s(h_1,h_2)}$.

In the previous section, we defined $h^*\in argmin_{h\in{\cal H}}{R(h)}$, $h^{*c}\in argmin_{h\in{\cal H}}{R^c(h)}$, $h^*_{DA}\in argmin_{h\in{\cal H}}{R_{DA}(h)}$, and $h^{*c}_{DA}\in argmin_{h\in{\cal H}}{{\mathbb E}_{(x,y)}{\mathbb E}_{x' \sim T(\cdot|x)}[\ell(h, (x’,y))]}$. Then, under the multi-domain setting, we assume the corresponding sets of optimal hypotheses: $\{h^*_1, \ldots , h^*_E \}$, $\{h^{*c}_1, \ldots , h^{*c}_E \}$, $\{ {h^*_{DA}}_1, \ldots , {h^*_{DA}}_E \}$, and $\{ {h^{*c}_{DA}}_1, \ldots , {h^{*c}_{DA}}_E \}$. This assumption is milder than assuming that there exists $h^*\in argmin {\cal R}(h)$. One important reason is that the latter requires $\cap_{s=1}^{E} argmin R_s(h) \ne \emptyset $, whereas assuming the existence of domain-wise optimal hypotheses $\{h^*_1, \ldots , h^*_E \}$ only requires $argmin R_s(h) \ne \emptyset, \forall s \in [E]$. In other words, the domain-wise optimal hypothesis assumption removes the cross-domain consistency requirement that all domains share a common minimizer. Consequently, it remains applicable even under substantial distribution shifts, where a shared optimal hypothesis may not exist.

Following previous studies~\citep{b6, b7}, we adopt the following assumption on the distribution family introduced in~\citep{b8}. 
\begin{assumption}
\label{assu2}
Assume a hypothesis space $ {\cal H} \subseteq \{ h:{\cal X} \to {\cal Y}\} $, and let $ {\cal P} = \left\{ {P \in {\cal P}({\cal H})} \right\}$ denote a family of data distributions, where $ {\cal P}({\cal H}) \subseteq {\cal P}({\cal X} \times {\cal Y})$.
\end{assumption}
\noindent\textbf{Remarks}: (a) This assumption essentially implies that the distribution family ${\cal P}$ is induced by functions in $\cal H$, while the marginal distribution $p(x)$ is assumed to be fixed. Consequently, each distribution $P \in {\cal P}$ can be identified with a hypothesis $h \in \cal H$, i.e., $P_{h} = p(x)h(y|x)$. (b) In learning theory, $ {\cal P}({\cal H}) $ is commonly used to denote the class of data distributions under which there exists a perfect predictor that lies exactly within the hypothesis space ${\cal H}$. In other words, for any distribution $ P \in {\cal P}({\cal H}) $, there exists some $ h^* \in {\cal H}$ such that $ h^*(x) $ perfectly predicts $ y $.

We further assume that any distribution set ${\bf{P}}$ defined in the previous subsection satisfies Assumption~\ref{assu2}. Consequently, for example, ${\bf{P}}^c=\{P^c_1, \ldots, P^c_E\}$ can equivalently be written as ${\bf{P}}^c=\{P^c_{h^{*c}_1}, \ldots, P^c_{h^{*c}_E}\}$. Here, we remind the reader that the original $P^c_1$ denotes a joint distribution over ${\cal X}$ and ${\cal Y}$. 

Similar to the previous section, we define the following noise-induced error terms:
\begin{align}
\varepsilon_{\eta} &= \frac{1}{E}\sum\nolimits_{i=1}^{E}{d_{P_i}(h^*_i, h^{*c}_i)} \nonumber \\
\varepsilon_{\eta_{DA}} &= \frac{1}{E}\sum\nolimits_{s=1}^{E}{d_{{P_{DA}}_s}({h^*_{DA}}_s, {h^{*c}_{DA}}_s)}. \nonumber 
\end{align}
Note that neither $P_i$ nor ${P_{DA}}_i$ is specified explicitly, as the above definitions hold for arbitrary domains sampled from ${\bf P}$ and ${\bf P}^c$, or from ${\bf P}_{DA}$ and ${\bf P}^c_{DA}$. Moreover, we do not use $s$ to index the domains of $\varepsilon_{\eta}$, since we need to account for noise-induced errors in both source and target domains. Specifically, for the source domains, we denote the corresponding noise-induced errors by $\varepsilon_{\eta}$, where $P_i \sim {\bf{P}}$ and the domain index is $s$. For the target domains, we denote the corresponding noise-induced errors by $\varepsilon_{\eta_t}$, where $P_i \sim {\bf{P_t}}$ and the domain index is $t$. For $\varepsilon_{\eta_{DA}}$, we only consider source domains, since DA is applied exclusively to source domains and not to target domains. Consequently, for target domains with $P_i \sim {\bf{P_t}}$, the corresponding noise-induced error remains denoted by $\varepsilon_{\eta_t}$. 

Firstly, we present the following result in terms of relative risks, similar to the previous section. 
\begin{theorem}
\label{theo8}
Given any multi-domain training set, the definition of the relative risks, and assuming $\ell \in [0,M]$ that is $\sigma$-admissible, for any $h_1, h_2 \in {\cal H}$, the following inequality holds with probability at least $1-\delta$: 
\begin{align}
\hat\Gamma(h_1,h_2)-\Gamma(h1, h_2)\le 4M\sqrt{\frac{\ln(1/\delta)}{2E \bar n}}. \nonumber 
\end{align}
\end{theorem}
\noindent\textbf{Remarks}: (a) The main order of this bound is ${\cal O}(1/\sqrt{E\bar n})$. (b) Here, $h_1$ and $h_2$ can be any hypotheses in ${\cal H}$. For the non-trivial case, we let $h_1$ and $h_2$ be different. If $h_1$ is equal to $h_2$, this leads to a trivial case, as the right-hand side (RHS) is always non-negative. 

Through the above bound, we observe that the tightness depends on $E$, rather than solely on $n$ as in Theorem~\ref{theo4p5}. This implies that as $E \to \infty$, the bound becomes tighter, which is consistent with conclusions in recent DG studies. Based on this result, we now present the following bounds. 
\begin{corollary}
\label{coro8}
Based on Theorem~\ref{theo8}, considering the training set ${\cal E}$ with $\eta$-noisy labels and the DA kernel $T$, for any $h \in {\cal H}$, we have the following inequalities:  
\begin{align}
\frac{1}{E\bar n}\sum\limits_{s=1}^E{\hat \gamma_s(h,h^*_s)} - \frac{1}{E}\sum\limits_{s=1}^E{\gamma_s(h,h^{*c}_s)} \le \sigma\varepsilon_{\eta} + 4M\sqrt{\frac{\ln(1/\delta)}{2E\bar n}}, \nonumber 
\end{align}
and 
\begin{align}
\frac{1}{E\bar n}\sum\limits_{s=1}^E{\hat \gamma_s(h,{h^*_{DA}}_s)} - \frac{1}{E}\sum\limits_{s=1}^E{\gamma_s(h,{h^{*c}_{DA}}_s)} \le \sigma\varepsilon_{\eta_{DA}} + 4M\sqrt{\frac{\ln(1/\delta)}{2E\bar n}}.  \nonumber 
\end{align}
Here, for any $s \in [E]$, ${h^{*c}}_s$, ${h^{*c}_{DA}}_s$, $\varepsilon_{\eta}$, and $\varepsilon_{\eta_{DA}}$ are defined above.
\end{corollary}
\noindent\textbf{Remarks}: (a) The main orders of the above bounds are ${\cal O}(1/\sqrt{E\bar n})$. (b) There are bias terms on the RHS of the above bounds, namely $\varepsilon_{\eta}$ and $\varepsilon_{\eta_{DA}}$, which may lead to loose or vacuous bounds even as $E \bar n \to \infty$. 

We also observe that, when comparing the two bounds, their difference additionally depends on the gap between $\varepsilon_{\eta}$ and $\varepsilon_{\eta_{DA}}$. Consequently, the following result holds. 
\begin{theorem}
\label{theo9}
By the definitions of $\varepsilon_{\eta}$ and $\varepsilon_{\eta_{DA}}$, the inequality follows: 
\begin{align}
\label{theo9eq1}
\left| \varepsilon_{\eta} - \varepsilon_{\eta_{DA}}\right|\le \frac{1}{E}\sum\limits_{s=1}^{E}{ \left(d_{x,x'}(h^*_s,{h^*_{DA}}_s)+ d_{x,x'}(h^{*c}_s,{h^{*c}_{DA}}_s) \right)} 
\end{align}
Here, $d_{x,x’}(h^*_s,{h^*_{DA}}_s):={\mathbb E}_{x,x’}[|h^*_s(x)-{h^*_{DA}}_s(x’)|]$ and $d_{x,x’}(h^{*c}_s,{h^{*c}_{DA}}_s):={\mathbb E}_{x,x’}[|h^{*c}_s(x)-{h^{*c}_{DA}}_s(x’)|]$.
\end{theorem}
\noindent\textbf{Remarks}: (a) Despite considering multiple-domain settings, observations similar to those from Inequality~\eqref{theo6eq1} can still be obtained. (b) The trivial condition is stronger than the i.i.d. condition, since for any domain $s \in [E]$, $d_{P_s}$ and $d_{{P_{DA}}_s}$ are required to be $0$. 

According to the above result, Theorems~\ref{theo7} and \ref{theo7p5} can also be extended to the distribution-shift setting. The proof is analogous and is therefore omitted. We directly state the following conditions for analyzing whether the noisy labels are propagated, 
\begin{align}
\frac{1}{E}\sum_s^E{\frac{\Delta_{\|x’-x\|}}{r(x)}} \ge \frac{\sigma\eta_{max}+\epsilon}{M}&, \nonumber \\
\frac{1}{E}\sum_s^E{d_{x’}(h,{h^*_{DA}}_s)-d_{x}(h,h^*_s)} \ge& \frac{1}{E}\sum_s^E{{\cal K}_s}, \nonumber 
\end{align}
where ${\cal K}_s$ represents the ${\cal K}$ (see the previous section) within each domain $s$. Now, based on the above results, we conclude that the risks induced by DA under i.i.d. data persist under distribution shifts. 

In fact, from the perspective of DG, a primary concern is the generalization error between the source and target domains. The above analyses are conducted solely on the source domains, which limits their applicability to DG settings. Therefore, to provide a more complete understanding of DA under DG, we further analyze the source-target generalization error. Consequently, we derive the following generalization bounds. 
\begin{theorem}[Generalization Bound]
\label{theo10}
Given the training set ${\cal E}$ and the testing set ${\cal E}_t$ with $\eta$-noisy labels, and assuming $\ell \in [0,M]$ that is $\sigma$-admissible, the generalization gap between the relative risks on ${\cal E}$ and ${\cal E}_t$ satisfies the following bound with probability at least $1-\delta$, for any $h \in {\cal H}$: 
\begin{align}
\frac{1}{E_t}\sum_t{\gamma(h,h^*_t) } \le \frac{1}{E \bar n}\sum_s{\hat \gamma(h,h^*_s)} + \sigma \left(\varepsilon_{\eta_t} + \varepsilon_{\eta} \right)+ C\left(\lambda^{*c} + \lambda \right ) + 4M\sqrt{\frac{\ln(1/\delta)}{2E\bar n}}. \nonumber 
\end{align}
Here, $\varepsilon_{\eta_t} = \frac{1}{E}\sum_{t=1}^{E_t}{d_{P_i}(h^*_t, h^{*c}_t)}$, $C = (2\sqrt{2}M)/{E E_t}$ , $\lambda = \sum_t\sum_s d_{JS}(P_{h^*_t},P_{h^*_s})$, and ${\lambda}^{*c} = \sum_t\sum_s d_{JS}(P_{h^{*c}_t}, P_{h^{*c}_s})$. 
\end{theorem}
\noindent\textbf{Remarks}: (a) The main orders of the above bound is ${\cal O}(1/\sqrt{E\bar n})$. When all $\varepsilon$ and $\lambda$ terms are assumed to be small, the bound is tight provided that $E\bar n \to \infty$. (b) By considering the JS divergence (JS distance) instead of the KL divergence to measure the discrepancy between distributions, the assumption that $P_i \ll P_j, \forall i \ne j \in [E]$, is no longer required. Moreover, $d_{JS}$ is bounded. (c) $\hat \gamma$ denotes the empirical risk, indicating that an ERM minimizer can make this term relatively small. (d) $\lambda^{*c}$ and $\lambda$ characterize the challenges of DG and are analogous to the assumptions in ~\citep{b5}. 

According to the above bound, we observe that the term $\varepsilon_{\eta_t}$ appears on the RHS. This is natural, as we do not assume that the target domains are free of label noise. Therefore, the sum of $\varepsilon_{\eta_t}$ and $\varepsilon_{\eta}$ contributes to the label-noise bias in the bound, which exhibits a significant difference from the i.i.d. data scenarios. Moreover, it is important to consider whether $\lambda^{*c}$ and $\lambda$ are sufficiently small. If either quantity is large, then, following the arguments in~\cite{b5}, generalization to the target domain becomes difficult; conversely, smaller values are more favorable for generalization. These uncontrollable factors ($\varepsilon$ and $\lambda$) on the RHS indicate that generalization under noisy data is inherently challenging. In particular, they are generally unavailable during training. We now argue that the above bound characterizes a more realistic and general DG setting, as label noise is often unavoidable in practical data collection processes. Based on this bound, we further derive the corresponding bound when applying $T$, as follows. 
\begin{corollary}
\label{coro9}
Given the training set ${\cal E}$ and the test set ${\cal E}_t$ with $\eta$-noisy labels, the DA kernel $T$, and assuming $\ell \in [0,M]$ that is $\sigma$-admissible, with probability at least $1-\delta$, the generalization gap between the relative risks on ${\cal E}$ after applying $T$ and on ${\cal E}_t$ is bounded as follows for any $h \in {\cal H}$: 
\begin{align}
\frac{1}{E_t}\sum_t{\gamma(h,h^*_t) } \le &\frac{1}{E \bar n}\sum_s{\hat \gamma(h,{h^*_{DA}}_s)} + \sigma \left(\varepsilon_{\eta_t} + \varepsilon_{\eta_{DA}} \right) \nonumber \\
&+ C\left(\lambda^{*c}_{DA} + {\lambda}_{DA} \right ) + 4M\sqrt{\frac{\ln(1/\delta)}{2E\bar n}}. \nonumber 
\end{align}
Here, ${\lambda}_{DA} = \sum_t\sum_s d_{JS}(P_{h^*_t},P_{{h^*_{DA}}_s})$ and ${\lambda}^{*c}_{DA} = \sum_t\sum_s d_{JS}(P_{h^{*c}_t}, P_{{h^{*c}_{DA}}_s})$. 
\end{corollary}
\noindent\textbf{Remarks}: (a) The same conclusions can be drawn, as discussed in Theorem~\ref{theo10}. ((b) Since $T$ is applied only to the source domains, the notations associated with the target domain remain the same as those in Theorem~\ref{theo10}. (c) Assuming that $\lambda_{DA}$ and $\lambda^{*c}_{DA}$ are small is no longer suitable, which requires further discussion.  

Here, $\lambda_{DA}$ and $\lambda^{*c}_{DA}$ are the new problems introduced by the DA methods. Compared with $\lambda$ and $\lambda^{*c}$, which typically represent the challenges introduced by the data themselves, $\lambda_{DA}$ and $\lambda^{*c}_{DA}$ are now linked to the learning strategies, bringing more uncertainty to the generalization problem. Is it necessary to derive conditions under which $\lambda_{DA} + \lambda^{*c}_{DA} \ge \lambda + \lambda^{*c}$ holds? The answer is negative, since such a derivation introduces additional unavailable terms, such as $h^*_t, h^{*c}_t, \forall t \in [E_t]$, which would require additional assumptions, or even stronger assumptions, for the analysis. Then, here we just need to know that applying DA methods further increases the uncertainty associated with learning and generalization. Consequently, we argue that DA methods should not be applied in the presence of noisy data for DG. Additionally, we establish the following result to support this argument.
\begin{theorem}
\label{theo11}
Given any training set ${\cal E}$ and the corresponding distribution set ${\bf P}$, applying the DA kernel $T$ to ${\cal E}$ yields the following inequality: 
\begin{align}
\label{theo11eq1}
JSD({P_{DA}}_1,\ldots,{P_{DA}}_E) \le JSD(P_1,\ldots,P_E) 
\end{align}
with equality if $T$ is $1\text{-}1$. 
\end{theorem}
\noindent\textbf{Remarks}: (a) Applying $T$ to the data in ${\cal E}$ leads to a decrease in information, as $JSD=I(\cdot;x)$, where $I$ represents mutual information. (b) If $T$ is $1\text{-}1$, there must exist $T^{-1}$ such that $x=T^{-1}(x’)$. 

From another perspective, applying $T$ essentially pulls the data distributions across source domains closer together. Note that we do not specifically assume that the data contain noisy labels, which implies that the above result is applicable to any data condition. This implies that DA methods reduce the distribution discrepancy, which may be regarded as an advantage. However, the following result illustrates the potential risk associated with this effect. 
\begin{theorem}
\label{theo12}
Let ${\bf P} = \{ P_1, \ldots, P_E \}$ be a finite set of probability distributions satisfying Assumption~\ref{assu2}, and let $JSD(P_1, \ldots, P_E)=K$. Then, for any $h \in {\cal H}$,  and any $epsilon > 0$, it holds that 
\begin{align}
\frac{1}{E}\sum_i^E{ P[ d_{P_i}(h,h^*_i) \ge \epsilon ] } \ge 1-\frac{K+\log(2)}{\log(E)}. \nonumber 
\end{align}
Here, $h^*_i \in argmin_{h\in {\cal H}} R_i(h), \forall i \in [E]$. 
\end{theorem}
\noindent\textbf{Remarks}: (a) There is a critical condition given by $K = \log(E/2)$. This condition implies that whenever $K \le \log(E/2)$, the above lower bound characterizes a regime in which model estimation is intrinsically difficult. In contrast, if $K > \log(E/2)$, the lower bound becomes vacuous, indicating that the estimation difficulty characterized by the bound no longer exists. (b) $JSD(P_1,\ldots,P_E) \le \log(E)$. (c) As $E \to \infty$, a small value of $K$ implies that the left-hand side (LHS) approaches $1$, which represents the most challenging scenario for estimation.  

The above lower bound is a variation of the well-known Fano’s lemma, and it clearly implies the risk associated with reducing the distribution discrepancy. Then, from Theorem~\ref{theo11}, we have $K_{DA} \le K$, which results in $LHS_{DA} \ge LHS$. One particular case is that $K_{DA} \ge \log(E/2)$, which implies that reducing the distribution discrepancy is acceptable. One may then ask what the essential effect of $LHS_{DA} \ge LHS$ is. We present the following result to answer this question.
\begin{theorem}
\label{theo13}
Given a training set ${\cal E}$ and a testing set ${\cal E}_t$, and assuming $\ell \in [0, M]$ that is $\sigma$-admissible, the generalization gap between the relative risks on ${\cal E}$ and ${\cal E}_t$ satisfies the following bound, for any $h \in {\cal H}$: 
\begin{align}
\frac{1}{E_t}\sum_t{\gamma(h,h^*_t)} \le \sigma \frac{1}{E}\sum_s{d_{P_s}(h,h^*_s)} + 2C \lambda. \nonumber 
\end{align}
Here, $h^*_s \in argmin_{h\in {\cal H}} R_s(h) ,\forall s \in [E]$ and $h^*_t \in argmin_{h\in {\cal H}} R_t(h) ,\forall t \in [E_t]$. 
\end{theorem}
\noindent\textbf{Remarks}: This bound is not of high-probability type, and its tightness is strongly dependent on the values of $d_{P_s}$ for all $s \in [E]$. Even when $\lambda$ is sufficiently small, large values of $d_{P_s}$ can render the bound loose or even vacuous. 

Since $P \propto d_P$, when $LHS_{DA} \ge LHS$, it follows that $d_{{P_{DA}}_s} \ge d_{P_s}$, which implies that the above bound becomes looser when applying $T$. Moreover, since we do not explicitly consider noise conditions, the above result holds under arbitrary data scenarios. In other words, regardless of label noise, applying DA methods can make the above bound looser, leading to less stable generalization. We now complete the analysis and summarize as follows. DA methods not only increase uncertainty in the learning process but also reduce the stability of generalization. The risk of using DA is greater in DG than in i.i.d. settings.

\section{Discussion}
\label{sec5}

\noindent\textbf{(1) $\sigma$-admissible and $M$}.  One may note that in many results, we only provide the notations $\sigma$ and $M$ to represent the Lipschitz property and the uniform bound of the loss function. This is because we do not specify the exact loss formulation or the specific task. In this section, we present the following examples for classification and regression tasks, and the corresponding lemma can be found on Page 515 in~\citep{b1}.
If ${\cal Y} \in \{-1, 1\}$, and considering the following loss function 
\begin{align}
\ell(f, (x,y))=(1-yf(x))_+=
\begin{cases}
1-yf(x), & if \; 1-yf(x)\ge 0 \\
0, & otherwise 
\end{cases}, \nonumber 
\end{align}
we have $\sigma = 1$ and $M=1$. If ${\cal Y} \in \{0, B\}$ and considering the following loss function 
\begin{align}
\ell(f, (x,y)) = (f(x)-y)^2, \nonumber 
\end{align}
we have $\sigma = 2B$ and $M = B$. 
Note that, for special learning objectives beyond ERM, such as Structure Risk Minimization (SRM) and Invariant Risk Minimization (IRM), $M$ is associated with the corresponding hyper-parameters of the objective.

\noindent\textbf{(2) Non-trivial Sufficient Condition}. We first discuss the assumption that $h(x) \in [0, \upsilon]$. As remarked previously, this assumption is mild for classification tasks, since most classification tasks can be formulated as binary classification problems. Under this formulation, we have $h(x) \in [0,1]$. Considering regression tasks, however, this assumption is somewhat strong, since the outputs may be unbounded, i.e., $h(x) \in [-\infty, +\infty]$. Nevertheless, there exist methods for mapping these outputs into a finite interval without losing their original meaning by shrinking their amplitudes. For example, we can introduce the following operator, the sigmoid transformation 
\begin{align}
s(y) = \frac{1}{1+e^{-y}}. \nonumber 
\end{align}
Through this transformation, it follows that $h(x) \in [0, 1]$. 

Then, we discuss why the presented conditions are considered loose. By examining the proof, one can observe that our original goal is to guarantee that $d_{x}(h,h^*)-d_{x’}(h,h^*_{DA}) \ge d_{x}(h,h^{*c}) + d_{x’}(h,h^{*c}_{DA})$. However, in practical applications, $h^{*c}$ and $h^{*c}_{DA}$ are unavailable. Therefore, we seek an upper bound $D$ for $d_{x}(h,h^{*c}) + d_{x’}(h,h^{*c}_{DA})$, i.e., $d_{x}(h,h^*)-d_{x’}(h,h^*_{DA})  \ge D \ge d_{x}(h,h^{*c}) + d_{x’}(h,h^{*c}_{DA})$. Nevertheless, since $h^{*c}$ and $h^{*c}_{DA}$ are inaccessible, we further derive a relatively loose upper bound ${\cal K}$. Consequently, the resulting sufficient conditions are also loose. If a tighter upper bound for $d_{x}(h,h^{*c}) + d_{x’}(h,h^{*c}_{DA})$ can be derived, the corresponding sufficient conditions can be tightened accordingly. We leave this as future work.
 
\noindent\textbf{(3) $\lambda$}. According to the bounds provided in Section~\ref{sec4}, we denote $\lambda$ as a measure of the challenge of Domain Generalization. This interpretation was first introduced by~\cite{b5}, who derived the following upper bound for domain adaptation: 
\begin{align}
\epsilon_T(h)\le \epsilon_S(h)+Cd_{{\cal H}\Delta{\cal H}}(S,T)+{\cal O}(1/\sqrt{m})+\lambda \nonumber .
\end{align}
Here, $\epsilon_T(h)$ and $\epsilon_S(h)$ represent the expected errors on the target domain $T$ and source domain $S$, $d_{{\cal H}\Delta{\cal H}}$ denotes the ${\cal H}$-divergence, and $\lambda=\epsilon_T(h^*)+\epsilon_S(h^*)$. Then, they argued that when $\lambda$ is large, one cannot expect to learn a good target classifier by minimizing the source error. Essentially, $\lambda = \sum_t\sum_s d_{JS}(P_{h^*_t},P_{h^*_s})$ is similar to theirs, as both depend on the assumed optimal hypotheses. However, they assume that the source and target domains share a common optimal hypothesis $h^*$, which is a strong assumption in DG. In contrast, our analysis adopts a milder assumption. Therefore, as long as $\lambda$ is sufficiently small, one can learn a well-generalized $\hat h$ across all domains.

\noindent\textbf{(4) $JSD({\bf P_{DA}})$ v.s. $JSD({\bf P})$}. In Section~\ref{sec4}, we present Inequality~\eqref{theo11eq1} as a simple characterization of the discrepancy between ${\bf P}$ and ${\bf P}_{DA}$. This is because the inequality can be derived without imposing any additional assumptions and follows directly from information-theoretic arguments. In fact, under the additional assumption that $v:=\inf_{x_i\in supp{P_i},x_j \in supp{P_j}, \forall i \ne j \in [E]}\|x_i-x_j\|, \; v > 0$, and by defining ${\mathbb E}[\Delta] := {\mathbb E}_{P \sim {\bf P}}[\Delta_{\|x’-x\|}]$, we further have 
\begin{align}
JSD({P_{DA}}_1,\ldots,{P_{DA}}_E) \approx JSD(P_1,\ldots,P_E)\cdot(1- \left(\frac{{\mathbb E}[\Delta]}{v} \right)^2).  \nonumber 
\end{align}
This relationship more clearly illustrates the effects of DA methods. If ${\mathbb E}[\Delta] / v \to 0$, or if no DA method is applied, the information remains unchanged. If $0 < {\mathbb E}[\Delta] / v < 1$, then we have $JSD_{DA} < JSD$. If ${\mathbb E}[\Delta] / v \to \infty$, we have $JSD_{DA} = 0$. In simple terms, DA changes the original support sets of the distributions and mixes them into indistinguishable states when ${\mathbb E}[\Delta] / v \to \infty$. Although this relationship facilitates the analysis, the assumption on $v$ is restrictive, as it characterizes the interactions among all distribution support sets. In practice, such an assumption may not be satisfied. 

\section{Conclusion}
\label{sec6}

According to analyses, we conclude that the potential risk of applying DA methods under label noise may outweigh their potential benefits. Although DA methods may yield $\varepsilon_{\eta_{DA}} \le \varepsilon_{\eta}$ with a certain probability, and $\Delta_{\|x'-x\|}$ can be controlled during training, the resulting uncertainty cannot be eliminated. Note that we do not advocate minimizing $|h(x')-h(x)|$ or $\left| |h(x)-y|-|h(x')-\tilde y| \right|$ during training. In other words, our objective is not to control $r(x)$, which is primarily constrained by the predictive capability of the model itself. 

Finally, we discuss a limitation of our work, namely the lack of empirical validation of the proposed sufficient conditions under both the i.i.d. setting and the DG scenario. At first glance, this limitation may hinder the practical applicability of our conditions in real-world settings. However, our analyses in the i.i.d. setting are motivated by a previous study~\citep{b13} that conducted a series of empirical validations. Our theorems provide a complete explanation for all the findings reported in that work under the i.i.d. setting. As for the DG scenario, we leave the empirical investigation to interested readers, as conducting such experiments would require substantial engineering effort and extensive computational resources beyond those currently available to us. 

\section*{Acknowledgements}





\appendix
\section{Lemmas \& Proofs}
\label{appendixlepr}

In this appendix, we provide the lemmas and related proofs that support our main theorems. 

\subsection{Lemmas}
\label{applemmas}

\begin{lemma}[Expectation difference bound via KL]
\label{lemm1}
Let $P$ and $Q$ be two probability distribution such that $P \ll Q$. Let $f: {\cal Z} \to {\mathbb R}$ be a measurable function satisfying 
\begin{align}
\left\|f\right\|_{\infty}:=\mathop{\sup}\limits_{z\in{\cal Z}}{\left| f(z) \right|}\le M \le \infty. \nonumber
\end{align} 
Then
\begin{align}
\left| {\mathbb E}_P[f]-{\mathbb E}_Q[f]\right|\le M\sqrt{2KL(P||Q)}. \nonumber 
\end{align}
\end{lemma} 
\begin{proof}
The result follows from standard inequalities in information theory. First, by a standard total variation bound, 
\begin{align}
\left| {\mathbb E}_P[f]-{\mathbb E}_Q[f]\right|\le 2\left\|f\right\|_{\infty}d_{TV}(P,Q). \nonumber 
\end{align}
Applying Pinsker's inequality, we obtain the desired result.

Here, we clarify how this lemma can be used to derive some results in our paper. Let $z=(x,y)$ and consider $\ell(h,z)$ as a function $f(z)$, since fixing $h$ makes $\ell$ independent of the underlying distributions $P$ and $Q$. Then, assuming $\|\ell\|_{\infty} \le M$, we have 
\begin{align}
{\mathbb E}_P[|\ell(h,z_1)|]-{\mathbb E}_Q[|\ell(h,z_2)|]\le M \sqrt{2KL(P||Q)}. \nonumber 
\end{align}
Here, $|a|-|b| \le |a - b|$. Note that if $h$ affects the distributions, we obtain 
\begin{align}
{\mathbb E}_P[|\ell(h_1,z_1)|]-{\mathbb E}_Q[|\ell(h_2,z_2)|]\le M \sqrt{2KL(P_{h_1}||Q_
{h_2})}. \nonumber 
\end{align}
\end{proof}

\begin{lemma}[Expectation difference bound via $W_1$]
\label{lemm2}
Let $f:{\cal Z} \to {\mathbb R}$ be an $L$-Lipschitz function, i.e.,
\begin{align}
\forall z_1,z_2 \in {\cal Z},\left|f(z_1)-f(z_2)\right|\le Ld(z_1,z_2). \nonumber
\end{align}
Then for any probability measures $P$, $Q$ on ${\cal Z}$,
\begin{align}
\left| {\mathbb E}_P[f]-{\mathbb E}_Q[f]\right|\le LW_1(P,Q). \nonumber 
\end{align}
\end{lemma}
\begin{proof}
Let $\gamma \in \prod(P,Q)$ be any coupling. Then 
\begin{align}
{\mathbb E}_P[f]-{\mathbb E}_Q[f]={\mathbb E}_{(z_1,z_2)\sim \gamma}[f(z_1)-f(z_2)]. \nonumber 
\end{align}
By the $L$-Lipschitz property of $f$, we have $\left| f(z_1)- f(z_2)\right| \le Ld(z_1,z_2)$, which yields
\begin{align}
\left| {\mathbb E}_P[f]-{\mathbb E}_Q[f]\right|\le L{\mathbb E}_{\gamma}[d(z_1,z_2)]. \nonumber 
\end{align} 
Taking the infimum over all couplings completes the proof. 

Here, we clarify how this lemma can be used to derive some results in our paper. Let $z=(x,y)$ and regard $\ell(h,z)$ as a function $f(z)$, since fixing $h$ makes $\ell$ independent of the underlying distributions $P$ and $Q$. Then, we have 
\begin{align}
{\mathbb E}_P[|\ell(h,z_1)|]-{\mathbb E}_Q[|\ell(h,z_2)|]\le \sigma {\mathbb E}_{\gamma}d(z_1,z_2). \nonumber 
\end{align}
Note that when $h$ influences the distributions, we instead consider the KL divergence or JS divergence. 
\end{proof}

\begin{lemma}[Hoeffding Inequality~\citep{b2}]
\label{lemm3}
Let $Z_1,\ldots,Z_n$ be $n$ i.i.d random variables with $Z_i \in [a, b]$ for each $i \in [n]$.  Then for any $\varepsilon > 0$, we have 
\begin{align}
P\left\{ \frac{1}{n}\sum\limits_{i=1}^{n}{Z_i}-{\mathbb E}[Z]\ge \varepsilon \right\}\le exp\left\{-\frac{2n\varepsilon^2}{(b-a)^2}\right\}. \nonumber 
\end{align}
\end{lemma}
\begin{proof}
This is a classical inequality, and thus we omit the proof here. We only illustrate how it is applied in our cases. We consider $\ell(h, z_i)$ as the $i$-th random variable and assume that $\ell \in [0,M]$. Then, based on this inequality, we obtain 
\begin{align}
P\left\{ \frac{1}{n}\sum\limits_{i=1}^{n}{\ell(h, z_i)} - {\mathbb E}_{z \sim P}[\ell(h, z)]\ge \varepsilon \right\}\le exp\left\{-\frac{2n\varepsilon^2}{M^2}\right\} \nonumber 
\end{align}
\end{proof}

\begin{lemma}[Hoeffding Inequality under Non-i.i.d. Settings] 
\label{lemm4}
Let $Z_1,\ldots, Z_n$ be n independent random variables with $Z_i \in [a_i, b_i]$ for each $i \in [n]$. Define $S_n=\sum\nolimits_{i=1}^{n}{Z_i}$. Then for any $\varepsilon > 0$, 
\begin{align}
P\left\{ S_n - {\mathbb E}[S_n]\ge \varepsilon \right\}\le exp\left\{-\frac{2\varepsilon^2}{\sum\nolimits_{i=1}^{n}(b_i-a_i)^2}\right\}. \nonumber 
\end{align}
\end{lemma}
\begin{proof}
This is a classical inequality, and thus we omit the proof here. We only illustrate how it is applied in our cases. We consider $Z_i=\ell(h, z_i)$ as the $i$-th random variable and $S_n=\sum\nolimits_{i=1}^{n}{Z_i}$, and assume that $Z_i \in [0, M_i]$. Then, based on this inequality, we have 
\begin{align}
P\left\{ S_n - {\mathbb E}[S_n]\ge \varepsilon \right\}\le exp\left\{-\frac{2\varepsilon^2}{\sum\nolimits_{i=1}^{n}M_i^2}\right\}.  \nonumber 
\end{align}
Note that no assumption is made that the variables are i.i.d. 
\end{proof}

\begin{lemma}[Expectation Bound via JS divergence]
\label{lemm5}
For any function $f$ satisfying the condition in Lemma~\ref{lemm1}, the difference between its expectations under $P$ and $Q$ is upper bounded by: 
\begin{align}
\left| {\mathbb E}_P[f]-{\mathbb E}_Q[f]\right|\le 2M\sqrt{2JS(P||Q)}, \nonumber 
\end{align}
or 
\begin{align}
\left| {\mathbb E}_P[f]-{\mathbb E}_Q[f]\right|\le 2\sqrt{2}M d_{JS}(P,Q). \nonumber 
\end{align}
\end{lemma}
\begin{proof}
The result follows from standard inequalities in information theory. First, by a standard total variation bound, 
\begin{align}
\left| {\mathbb E}_P[f]-{\mathbb E}_Q[f]\right|\le 2\left\|f\right\|_{\infty}d_{TV}(P,Q). \nonumber 
\end{align}
Applying the Pinsker's inequality, we have 
\begin{align}
\left(\frac{1}{2}d_{TV}(P,Q)\right)^2 &\le \frac{1}{2}KL(P||M) \nonumber \\
\left(\frac{1}{2}d_{TV}(P,Q)\right)^2 &\le \frac{1}{2}KL(Q||M). \nonumber 
\end{align}
Here, $M=(P+Q)/2$, $d_{TV}(P,M)=(1/2)d_{TV}(P,Q)$, and $d_{TV}(Q,M)=(1/2)d_{TV}(P,Q)$. Then, according to Definition~\ref{defi4}, i.e., $JS(P;Q)=1/2KL(P;M)+1/2KL(Q;M)$, we obtain 
\begin{align}
2\left(\frac{1}{2}d_{TV}(P,Q)\right)^2 \le \frac{1}{2}JS(P||Q). \nonumber 
\end{align}

Here, we clarify how this lemma can be used to derive some results in our paper. Let $z=(x,y)$ and consider $\ell(h,z)$ as a function $f(z)$, since fixing $h$ makes $\ell$ independent of the underlying distributions $P$ and $Q$. Then, assuming $\|\ell\|_{\infty} \le M$, we have 
\begin{align}
{\mathbb E}_P[|\ell(h,z_1)|]-{\mathbb E}_Q[|\ell(h,z_2)|] \le 2M \sqrt{2JS(P||Q)}. \nonumber 
\end{align}
Note that if $h$ affects the distributions, we obtain 
\begin{align}
{\mathbb E}_P[|\ell(h_1,z_1)|]-{\mathbb E}_Q[|\ell(h_2,z_2)|] \le 2M \sqrt{2JS(P_{h_1}||Q_
{h_2})}. \nonumber 
\end{align}
\end{proof}

\begin{lemma}[Properties of Divergences~\citep{b9}]
\label{lemm6}
(1) $x’=f(x) \Rightarrow KL(P_x||Q_x)\ge KL(P_{x’}||Q_{x’})$, with equality if $f$ is $1\text{-}1$. 
(2) $x’=f(x) \Rightarrow KL(P_{(x,y)}||Q_{(x,y)})\ge KL(P_{(x’,y)}||Q_{(x’,y)})$, with equality if $f$ is $1\text{-}1$.
\end{lemma}
\begin{proof}
Lemma (1) is a direct consequence of Theorem 2.2 in \citep{b9} (Page 15), whose proof is omitted therein. This lemma is a classical result in information theory, indicating that whenever the input is processed by a non-injective function, some information is inevitably lost, leading to a reduction in the KL divergence. We now proceed to prove Lemma (2). 

Let $z=(x,y)$, and define a deterministic mapping $g$ by $g(x,y)=(f(x),y)$. Denote $z'=g(z)=(x',y)$. Then, by the data processing inequality, it follows that $KL(P_z||Q_z)\ge KL(P_{z’}||Q_{z’})$. 

\end{proof}

\begin{lemma}[Fano’s Lemma]
\label{lemm7}
Let $S$ be a finite set of probabilities $S=\{ P_1,\ldots , P_m \}$ such that $JSD(P_1,\ldots , P_m) = K$. Then for any estimate $\phi$ with values in $\{1, \ldots, m\}$ we have 
\begin{align}
m^{-1}\sum_{i=1}^m{P[\phi \ne i]} \ge {1-\frac{K+\log(2)}{\log(m)}}. \nonumber 
\end{align}
\end{lemma}
\begin{proof}
This is also a classical Fano-type lemma in statistical learning theory (see the proof in~\citep{b10}). It characterizes a fundamental difficulty in statistical estimation: as the JSD between distributions decreases, the estimation error necessarily increases. 

In our work, we consider ${\cal P} = \{P \in {\cal P}({\cal H})\}$. Therefore, this lemma can be naturally extended to model estimation. Details are provided in the proofs of the corresponding theorems. 
\end{proof}

\subsection{Proofs}
\label{appproofs}

\begin{proof}{\textbf{of Theorem~\ref{theo1}.} }
Since ${\cal L}(\theta)={\mathbb E}_{(x,\tilde y)}[\ell(h,(x, \tilde y))] = {\mathbb E}_{(x,y)}{\mathbb E}_{\tilde y}[\ell(h,(x,\tilde y))]$, we have 
\begin{align}
{\cal L}(\theta)&={\cal L}^c(\theta)+{\mathbb E}_{(x,y)}{\mathbb E}_{\tilde y}[\ell(h,(x,\tilde y))-\ell(h,(x,y))] \nonumber \\
&={\cal L}^c(\theta)+\eta{\mathbb E}_{(x,y)}{\mathbb E}_{\tilde y \ne y}[\ell(h,(x,\tilde y))-\ell(h,(x,y))] \nonumber \\
&={\cal L}^c(\theta)+\eta\Delta(\theta) \nonumber
\end{align}
where $\Delta(\theta)={\mathbb E}_{(x,y)}{\mathbb E}_{\tilde y \ne y}[\ell(h,(x,\tilde y))-\ell(h,(x,y))] $. 

Let $g_k = \nabla{\cal L}(\theta_k)$ and $\nabla{\cal L}(\theta_t)=\nabla{\cal L}^c(\theta_k)+\eta\nabla\Delta(\theta_k)$. Here, we consider ${\nabla}_{\theta}{\cal L}(\theta)={\mathbb E}_{(x,y) \sim P}[{\nabla}_{\theta}\ell(h,(x,y))] $, and do not necessarily to explicitly Jacobian matrix $J_{\theta}(x)={\nabla_{\theta}}h_{\theta}(x)$, as it is not essential for the following derivations. Since the stochastic gradient $g_k$ is an unbiased estimator of the gradient for noisy samples, i.e., ${\mathbb E}[g_k \mid \theta_k]=\nabla{\cal L}(\theta_k)$, we have the following second moment 
\begin{align}
{\mathbb E}\|g_k\|^2 \le 2{\mathbb E}\|{\nabla}{\cal L}^c\|^2+2\eta^2{\mathbb E}\|{\nabla}{\Delta}\|^2. \nonumber 
\end{align}
By Definition~\ref{defi1}, we have $\|{\nabla}{\cal L}^c\|^2\le\sigma^2$ and obtain ${\mathbb E}\|g_k\|^2 \le 2(\sigma^2+\eta^2{\mathbb E}\|{\nabla}{\Delta}\|^2):=G^2(\eta)$.

Under Assumption~\ref{assu1}, let $\theta^*$ denote an optimal solution for ${\cal L}(\theta)$. Then, one step of SGD yields 
\begin{align}
\|\theta_{k+1}-\theta^*\|=\|\theta_k-\theta^*-\alpha_k g_k\|. \nonumber 
\end{align}
Squaring both sides, we obtain 
\begin{align}
\|\theta_{k+1}-\theta^*\|^2&=\|\theta_k-\theta^*-\alpha_k g_k\|^2 \nonumber \\
&=\|\theta_k-\theta^*\|^2-2\alpha_k\langle g_k,\theta_k-\theta^* \rangle+\alpha^2_kg^2_k .\nonumber 
\end{align}
Taking expectation and using the convexity of $\ell$, we have 
\begin{align}
{\mathbb E}[\|\theta_{k+1}-\theta^*\|^2 \mid \theta_k] &\le {\mathbb E}[\|\theta_k-\theta^*\|^2] -2\alpha_k{\mathbb E}[{\cal L}(\theta_k)-{\cal L}(\theta^*)]+\alpha^2_k {\mathbb E}\|g_k\|^2 \nonumber \\
2\alpha_k{\mathbb E}[{\cal L}(\theta_k)-{\cal L}(\theta^*)] &\le {\mathbb E}[\|\theta_k-\theta^*\|^2]-{\mathbb E}[\|\theta_{k+1}-\theta^*\|^2] +\alpha^2_kG^2(\eta). \nonumber 
\end{align}
Then, by a telescoping argument, we obtain 
\begin{align}
2\sum\limits^{K}_{k=1}{\alpha_k{\mathbb E}[{\cal L}(\theta_k)-{\cal L}(\theta^*)]} \le \Theta^2 + G^2(\eta)\sum\limits^{K}_{k=1}{\alpha^2_k}, \nonumber
\end{align}
where $\Theta = \|\theta_0-\theta^*\|$. 
Let ${\bar\theta_K}={\sum\alpha\theta_k}/{\sum\alpha}$ and $\alpha=\Theta/(G(\eta)\sqrt{K})$, we finally have 
\begin{align}
{\mathbb E}[{\cal L}(\bar\theta_K)-{\cal L}(\theta^*)]\le \frac{\Theta G(\eta)}{\sqrt{K}}. \nonumber
\end{align}
Assume that $\theta^{*c} \in argmin{\cal L}^c(\theta)$, we have
\begin{align}
{\cal L}(\theta^*)-{\cal L}^c(\theta^{*c})&=\eta\Delta(\theta^*)+[{\cal L}^c(\theta^*)-{\cal L}^c(\theta^{*c})] \nonumber \\
&\le \eta\Delta(\theta^*) \nonumber \\
&\le \eta\sigma{\mathbb E}_{y,\tilde y}\|\tilde y- y\|, \nonumber 
\end{align}
where the last inequality holds for Definition~\ref{defi1}. Let $\Delta_{\| \tilde y- y \|}={\mathbb E}_{y,\tilde y}\|\tilde y- y\|$, we have 
\begin{align}
{\cal L}(\theta^*)-{\cal L}^c(\theta^{*c}) &\le \eta\sigma\Delta_{\|\tilde y-y\|} \nonumber \\
{\cal L}(\bar\theta_K)+{\cal L}(\theta^*)-{\cal L}^c(\theta^{*c}) &\le {\cal L}(\bar\theta_K)+\eta\sigma\Delta_{\|\tilde y-y\|} \nonumber \\
{\cal L}(\bar\theta_K)-{\cal L}^c(\theta^{*c}) &\le {\cal L}(\bar\theta_K)-{\cal L}(\theta^*)+\eta\sigma\Delta_{\|\tilde y-y\|}. \nonumber 
\end{align}
Combining the previous results, we have 
\begin{align}
{\mathbb E}[{\cal L}(\bar\theta_K)-{\cal L}^c(\theta^{*c})] &\le {\mathbb E}[{\cal L}(\bar\theta_K)-{\cal L}(\theta^*)]+\eta\sigma\Delta_{\|\tilde y-y\|} \nonumber \\
{\mathbb E}[{\cal L}(\bar\theta_K)-{\cal L}^c(\theta^{*c})] &\le\frac{\Theta G(\eta)}{\sqrt{K}} +\eta\sigma\Delta_{\|\tilde y-y\|}. \nonumber 
\end{align}
To achieve an $\varepsilon$-optimal solution, it suffices that 
\begin{align}
K\ge\left(\frac{\Theta G(\eta)}{\varepsilon-\eta\sigma\Delta_{\|\tilde y-y\|}}\right)^2, \nonumber 
\end{align}
which derives $\eta $. 
\end{proof}

\begin{proof}{\textbf{of Theorem~\ref{theo2}.} }
Since ${\cal L}_{DA}(\theta)={\mathbb E}_{(x,\tilde y)}{\mathbb E}_{x'}[\ell(h,(x',\tilde y))]$, we have 
\begin{align}
{\cal L}_{DA}(\theta)&={\cal L}(\theta)+{\mathbb E}_{(x,\tilde y)}{\mathbb E}_{x'}[\ell(h,(x',\tilde y))-\ell(h,(x,\tilde y))] \nonumber \\
&= {\cal L}(\theta)+\Delta_{DA}(\theta), \nonumber
\end{align}
where $\Delta_{DA}(\theta)={\mathbb E}_{(x,\tilde y)}{\mathbb E}_{x'}[\ell(h,(x',\tilde y))-\ell(h,(x,\tilde y))]$. 

Let $g_k = \nabla{\cal L}_{DA}(\theta_k)$ and $\nabla{\cal L}_{DA}(\theta_t)=\nabla{\cal L}(\theta_k)+\nabla\Delta_{DA}(\theta_k)$. Similarly, we have 
\begin{align}
{\mathbb E}\|g_k\|^2 \le 2{\mathbb E}\|{\nabla}{\cal L}\|^2+2{\mathbb E}\|{\nabla}{\Delta_{DA}}\|^2. \nonumber 
\end{align}
By Definition~\ref{defi1}, we have $\|{\nabla}{\cal L}\|^2\le\sigma^2$ and ${\mathbb E}\|g_k\|^2 \le 2(\sigma^2+{\mathbb E}\|{\nabla}{\Delta_{DA}}\|^2):=G^2_{DA}$.

Under Assumption~\ref{assu1} and by following the previous steps in the proof of Theorem~\ref{theo1}, we finally obtain 
\begin{align}
{\mathbb E}[{\cal L}_{DA}(\bar\theta_K)-{\cal L}_{DA}(\theta^*_{DA})]\le \frac{\Theta G_{DA}}{\sqrt{K}} ,\nonumber 
\end{align}
where $\theta^*_{DA}$ is an optimal solution for ${\cal L}_{DA}(\theta)$. 
Consider that $\theta^* \in argmin{\cal L}(\theta)$ and $\theta^{*c} \in argmin{\cal L}^c(\theta)$, we have
\begin{align}
{\cal L}_{DA}(\theta^*_{DA})-{\cal L}(\theta^{*})&=\Delta(\theta^*_{DA})+[{\cal L}(\theta^*_{DA})-{\cal L}(\theta^{*})] \nonumber \\
&\le \Delta(\theta^*_{DA}) \nonumber \\
&\le \sigma {\mathbb E}_{x,x’}\|x’-x\|, \nonumber 
\end{align}
where the last inequality holds for Lemma~\ref{lemm2}. Let $\Delta_{\|x'-x\|}={\mathbb E}_{x,x’}\|x’-x\|$, we have 
\begin{align}
{\cal L}_{DA}(\theta^*_{DA})-{\cal L}(\theta^{*}) &\le \sigma\Delta_{\|x’-x\|} \nonumber \\
({\cal L}_{DA}(\theta^*_{DA})-{\cal L}(\theta^{*}))+({\cal L}(\theta^{*})-{\cal L}^c(\theta^{*c})) &\le \sigma\Delta_{\|x'-x\|}+\eta\sigma\Delta_{\|\tilde y-y\|} \nonumber \\
{\cal L}_{DA}(\theta^*_{DA})-{\cal L}^c(\theta^{*c})&\le \sigma\Delta_{\|x'-x\|}+\eta\sigma\Delta_{\|\tilde y-y\|} \nonumber \\
{\cal L}_{DA}(\bar\theta_{K})+{\cal L}_{DA}(\theta^*_{DA})-{\cal L}^c(\theta^{*c})&\le {\cal L}_{DA}(\bar\theta_{K}) + \sigma\Delta_{\|x'-x\|}+\eta\sigma\Delta_{\|\tilde y-y\|} \nonumber \\
{\cal L}_{DA}(\bar\theta_{K})-{\cal L}^c(\theta^{*c})&\le {\cal L}_{DA}(\bar\theta_{K})-{\cal L}^*_{DA}+ \sigma\Delta_{\|x'-x\|}+\eta\sigma\Delta_{\|\tilde y-y\|}. \nonumber 
\end{align}
Combining the previous results, we have
\begin{align}
{\mathbb E}[{\cal L}_{DA}(\bar\theta_{K})-{\cal L}^c(\theta^{*c})]\le \frac{\Theta G_{DA}}{\sqrt{K}}+ \sigma\Delta_{\|x'-x\|}+\eta\sigma\Delta_{\|\tilde y-y\|}. \nonumber 
\end{align}
To achieve an $\varepsilon$-optimal solution, it suffices that 
\begin{align}
K\ge\left(\frac{\Theta G_{DA}}{\varepsilon-\eta\sigma\Delta_{\|\tilde y-y\|} - \sigma\Delta_{\|x'-x\|}} \right)^2, \nonumber 
\end{align}
which derives $\eta $. 
\end{proof}

\begin{proof}{\textbf{of Corollary~\ref{coro1}.} }
The proof process follows that of Theorem~\ref{theo2}. The main difference is that we additionally need to consider $\Delta'_{DA}:={\mathbb E}_{x,\tilde y}{\mathbb E}_{(x',\tilde y')}[\ell(h,(x',\tilde y'))-\ell(h,(x,\tilde y))]$, since we have ${\cal L}'_{DA}(\theta)={\cal L}(\theta)+{\cal L}'_{DA}(\theta)-{\cal L}(\theta)$. Assuming $\ell \le M$, by Lemma~\ref{lemm1}, we obtain 
\begin{align}
\Delta'_{DA} \le M\sqrt{2KL(P||P’_{DA})}, \nonumber 
\end{align}
which completes the proof.
\end{proof}

\begin{proof}{\textbf{of Theorem~\ref{theo3}.} }
We first establish the following result: 
\begin{align}
\hat R(h)-R(h)&=\eta(\hat R(h)-R(h))_{\tilde y \ne y}+(1-\eta)(\hat R(h) - R(h))_{\tilde y = y} \nonumber \\
&=\eta(\hat R(h)-R(h)_{\tilde y \ne y}+(1-\eta)(\hat R^c(h) - R^c(h)). \nonumber 
\end{align}
Since we assume that $\ell \le M$, it follows that $\ell(h,(x_i,\tilde y_i))-{\mathbb E}_{(x,\tilde y)}[\ell] \in [-M, M]$. Then, by Lemma~\ref{lemm3}, with probability at least $1-(\delta/2)$, we have 
\begin{align}
\hat R(h)-R(h)_{\tilde y \ne y} \le \sqrt{\frac{(2M)^2\ln(2/\delta)}{2n_{\eta}}}, \nonumber 
\end{align}
and 
\begin{align}
\hat R^c(h)-R^c(h)) \le \sqrt{\frac{(2M)^2\ln(2/\delta)}{2n_{c}}}, \nonumber 
\end{align}
which completes the proof. 
\end{proof}

\begin{proof}{\textbf{of Theorem~\ref{theo4}.} }
We first establish the following result: 
\begin{align}
\hat R_{DA}(h)-R(h)=[\hat R_{DA}(h)-\hat R(h)]+[\hat R(h)-R(h)]. \nonumber 
\end{align}
The second term has been addressed previously; thus, we focus on the first term. Let $Z_i = [\ell(h,(x’_i,\tilde y_i)) - \ell(h,(x_i,\tilde y_i))] - [{\mathbb E}_{(x’,\tilde y)} - {\mathbb E}_{(x,\tilde y)}]$, then it follows that $Z_i \in [-2M, 2M]$. Similarly, by Lemma~\ref{lemm3}, with probability at least $1-(\delta/3)$, we have 
\begin{align}
[\hat R_{DA}(h)-\hat R(h)]&-[R_{DA}(h)-R(h)] \le \sqrt{\frac{((4M)^2\ln(3/\delta)}{2n}}. \nonumber \\
&\le [R_{DA}(h)-R(h)] + \sqrt{\frac{(4M)^2\ln(3/\delta)}{2n}} \nonumber \\
&\le \sigma\Delta_{\|x’-x\|} + \sqrt{\frac{(4M)^2\ln(3/\delta)}{2n}}, \nonumber 
\end{align}
where the last inequality follows from Lemma~\ref{lemm2}, completing the proof. 
\end{proof}

\begin{proof}{\textbf{of Corollary~\ref{coro3}}.} 
Based on the proof of Corollary~\ref{coro1}, we have 
\begin{align}
R’_{DA}(h)-R(h) \le M\sqrt{2KL(P||P’_{DA})}, \nonumber 
\end{align}
which completes the proof by following the same procedure as in the proof of Theorem~\ref{theo4}.
\end{proof}

\begin{proof}{\textbf{of Corollary~\ref{coro2} and Corollary~\ref{coro4}.} }
Considering the events $\hat R(h) - R(h) \ge \varepsilon$, $\hat R_{DA}(h) - R(h) \ge \varepsilon$, and $\hat R'_{DA}(h) - R(h) \ge \varepsilon$, the desired results follow directly from the corresponding theorems. For example, $ \varepsilon \le \eta A+(1-\eta) B = \eta A - \eta B + B $ leads $\eta \ge \frac{\varepsilon - B}{A-B}$ and $\eta \le \frac{B - \varepsilon}{B - A}.$ 
\end{proof}

\begin{proof}{\textbf{of Theorem~\ref{theo4p5}.} }
Since $\ell$ is bounded by $M$, let $Z_i=\hat\gamma(h_1,h_2)_i-\gamma(h_1,h_2)$, we obtain $Z_i \in [-2M,2M]$. Then, by Lemma~\ref{lemm3}, for any $\varepsilon > 0$, we have 
\begin{align}
P\left\{ \frac{1}{n}\sum\limits_{i=1}^{n}{Z_i}-{\mathbb E}[Z]\ge \varepsilon \right\}\le exp\left\{-\frac{2n\varepsilon^2}{(4M)^2}\right\}. \nonumber 
\end{align}
\end{proof}

\begin{proof}{\textbf{of Corollary~\ref{coro4p5}.} }
Based on Theorem~\ref{theo4p5}, different types of inequalities can be derived. Here, we only illustrate how to decompose the relative expected risks. Specifically, we consider the following decomposition: 
\begin{align}
{\mathbb E}[\ell(h_1,z)-\ell(h_2,z)]&={\mathbb E}[\ell(h_1,z)+\ell(h,z)-\ell(h,z)-\ell(h_2,z)] \nonumber \\
&={\mathbb E}[\ell(h_1,z)-\ell(h,z)]+{\mathbb E}[\ell(h,z)-\ell(h_2,z)]. \nonumber 
\end{align}
Then, based on Definition~\ref{defi1} and Definition~\ref{defi3}, we rewire ${\mathbb E}[\ell(h,z)-\ell(h_2,z)]$ to $\sigma d(h, h_2)$, completing the proof. 
\end{proof}

\begin{proof}{\textbf{of Theorem~\ref{theo6}.} }
Let $a=h^*(x)$, $b=h^{*c}(x)$, $c=h^*_{DA}(x’)$, and $d=h^{*c}_{DA}(x’)$. Then, applying the inequality 
\begin{align}
\left| |a-b|-|c-d| \right|\le \left|a-c\right| + \left|b-d\right|, \nonumber 
\end{align}
the proof is completed by taking expectations on both sides. 
\end{proof}

\begin{proof}{\textbf{of Theorem~\ref{theo7}.} }
Since $\epsilon = {\mathbb E}_{x’,\tilde y}[\ell(h^*_{DA}, (x’,\tilde y))]$, it holds that 
\begin{align}
{\mathbb E}_{x',\tilde y}[\left|\ell(h^*_{DA}, (x’,\tilde y))-\ell(h^{*c}_{DA}, (x’,\tilde y))\right|]&={\mathbb E}_{x',\tilde y}[\left|\ell(h^{*c}_{DA}, (x’,\tilde y))-\epsilon\right|] \nonumber \\ 
&={\mathbb E}_{x',\tilde y}[\left|\ell(h^{*c}_{DA}, (x’,\tilde y))\right|] -\epsilon \nonumber
\end{align}
Moreover, since $\ell \in [0, M]$, the expected risk can be transformed into the error rate as 
\begin{align}
{\mathbb E}_{x',\tilde y}[\left|\ell(h^{*c}_{DA}, (x’,\tilde y))\right|] = M \cdot {P}_{DA}(h^{*c}_{DA}(x’)\ne\tilde y), \nonumber 
\end{align}
and let $P_n = {P}_{DA}(h^{*c}_{DA}(x’)\ne\tilde y)$ and $h=h^{*c}_{DA}$. Then, according to the Law of Total Probability, we can decompose $P_n$ as follows while ignoring higher-order infinitesimal terms: 
\begin{align}
P_n\approx \underbrace {P(\tilde y\ne y)P(h(x')=y|\tilde y \ne y)}_{\text{error by noise}}+\underbrace {P(\tilde y = y)P(h(x')\ne y|\tilde y = y)}_{\text{error by DA}}. \nonumber 
\end{align} 
Since $h$ is the optimal solution trained on cleaned data with DA, $P(h(x')=y | \tilde y \ne y)=1$, while $P(h(x')\ne y | \tilde y = y)$ is mainly determined by the prediction behavior of $h$, namely that $h$ fails to correctly classify the input $x'$. This implies that the feature perturbation $|x'-x|$ crosses the decision margin boundary of $h$, namely, $|x’-x|>r(x)$,  and we have 
\begin{align}
P(h(x')\ne y|\tilde y = y) = P_{r(x)}(\|x’-x\|\ge r(x)). \nonumber 
\end{align}
Then, we obtain 
\begin{align}
P_n=P(\tilde y\ne y) + P(\tilde y= y)P_{r(x)}. \nonumber 
\end{align}
Let $Z = |x' - x|$. By Markov's inequality, we have 
\begin{align}
P(Z \ge r(x)) \le \frac{{\mathbb E}[Z]}{r(x)}. \nonumber 
\end{align}
We reformulate all results as follows: 
\begin{align}
MP_n - \varepsilon \approx P(\tilde y\ne y) + P(\tilde y= y)P_{r(x)}. \nonumber 
\end{align}
Based on the above, we know that 
\begin{align} 
P_n \ge \max \left\{ P(\tilde y \ne y), P_{r(x)},\frac{1}{2}\left( P(\tilde y \ne y)+ P_{r(x)}\right)  \right\}. \nonumber  
\end{align}
Moreover, $P(\tilde y\ne y) \propto {\mathbb E}_{x’,\tilde y, y}[|\tilde y - y|]$, and let $P_{r(x)}={\mathbb E}[|x’-x|]/r(x)$, we have 
\begin{align}
P_n \ge \max \left\{ \underbrace {\alpha{\mathbb E}_{x’,\tilde y, y}[|\tilde y - y|]}_A, \underbrace{\frac{{\mathbb E}[|x’-x|]}{r(x)}}_B, \underbrace{\frac{1}{2}\left( \alpha {\mathbb E}_{x’,\tilde y, y}[|\tilde y - y|] + \frac{\mathbb E[|x’-x|]}{r(x)}\right)}_C  \right\}. \nonumber 
\end{align}
Here, ${\mathbb E}_{x’,\tilde y, y}$ and ${\mathbb E}$ denote the expectations with respect to $P_{DA}$, and $\alpha$ represents a constant discrepancy between $P$ and ${\mathbb E}$. Then, based on the condition~\eqref{theo7eq1} and $d_{x}(h^*,h^{*c}) \le \eta_{max}$, we have 
\begin{align}
\max\{A,B,C\} \ge \frac{\sigma \eta_{max} - \epsilon}{M}. \nonumber 
\end{align}
Namely, 
\begin{align}
M \max\{A,B,C\} - \epsilon \ge \sigma \eta_{max}. \nonumber 
\end{align}
Since $\ell$ is $\sigma$-admissible, it follows that $\sigma d_{x’}(h^*_{DA}, h^{*c}_{DA}) \ge {\mathbb E}_{x’,\tilde y}[| \ell(h^*_{DA},\cdot) -\ell(h^{*c}_{DA}, \cdot) |]$. Thus, we have 
\begin{align}
\sigma d_{x'}(h^*_{DA}, h^{*c}_{DA}) \ge M \cdot P_n - \epsilon \ge M \max\{A,B,C\} -\epsilon \ge \sigma \eta_{max} \ge \sigma d_{x}(h^*, h^{*c} ), \nonumber 
\end{align}
which completes the proof. 

\end{proof}

\begin{proof}{\textbf{of Theorem~\ref{theo7p5}.} }
By the triangle inequality, for any $h \in {\cal H}$, we have 
\begin{align}
-d_{x}(h^*,h^{*c})&\ge -d_{x}(h,h^*)-d_{x}(h,h^{*c}) \nonumber \\
-d_{x’}(h^*_{DA},h^{*c}_{DA})&\ge -d_{x’}(h,h^*_{DA})-d_{x’}(h,h^{*c}_{DA}) \nonumber 
\end{align}
First, we prove that $d_{x}(h,h^*)-d_{x’}(h,h^*_{DA})\ge{\cal K} \Rightarrow \varepsilon_{\eta} \ge \varepsilon_{\eta_{DA}}$. 
From the above, we have 
\begin{align}
d_{x}(h^*,h^{*c})-d_{x’}(h^*_{DA},h^{*c}_{DA})\ge& d_{x}(h^*,h^{*c})-d_{x’}(h,h^*_{DA})-d_{x’}(h,h^{*c}_{DA}) \nonumber \\
=&d_{x}(h^*,h^{*c}) +d_{x}(h^{*c},h) -d_{x}(h^{*c},h) \nonumber \\
&-d_{x’}(h,h^*_{DA})-d_{x’}(h,h^{*c}_{DA}) \nonumber \\
\ge& d_{x}(h^*,h)-d_{x}(h^{*c},h)-d_{x’}(h,h^*_{DA})-d_{x’}(h,h^{*c}_{DA}) \nonumber \\
=& d_{x}(h,h^*)-d_{x’}(h,h^*_{DA})-\left[ d_{x}(h,h^{*c}) + d_{x’}(h,h^{*c}_{DA}) \right] \nonumber 
\end{align} 
By the definition of $d$, we have 
\begin{align}
d_{x}(h,h^{*c}) + d_{x’}(h,h^{*c}_{DA}) &= {\mathbb E}_{x}[| h(x)-h^{*c}(x) |]+ {\mathbb E}_{x’}[| h(x’)-h^{*c}_{DA}(x’) |] \nonumber \\
& \le {\mathbb E}_{x}[| h(x)|]+{\mathbb E}_{x}[|h^{*c}(x) |]+ {\mathbb E}_{x'}[| h(x')|]+{\mathbb E}_{x'}[| h^{*c}_{DA}(x') |] \nonumber \\
& \le {\mathbb E}_{x}[| h(x)|]+ {\mathbb E}_{x'}[| h(x')|]+2\upsilon. \nonumber 
\end{align}
Here, the inequalities hold due to the triangle inequality and the assumption that $h(x) \in [0, \upsilon]$. Since $d_{x}(h,h^*)-d_{x’}(h,h^*_{DA}) \ge {\cal K}$, we obtain 
\begin{align}
d_{x}(h,h^*)-d_{x’}(h,h^*_{DA})-\left[ d_{x}(h,h^{*c}) + d_{x’}(h,h^{*c}_{DA}) \right] \ge 0, \nonumber 
\end{align}
that is $\varepsilon_{\eta} - \varepsilon_{\eta_{DA}} \ge 0$. 

Second, we prove that $d_{x’}(h,h^*_{DA})-d_{x}(h,h^*)\ge{\cal K} \Rightarrow \varepsilon_{\eta_{DA}} \ge \varepsilon_{\eta}$. Similarly, we have 
\begin{align}
d_{x'}(h^*_{DA},h^{*c}_{DA}) - d_{x}(h^*,h^{*c}) &\ge d_{x'}(h^*_{DA},h^{*c}_{DA})-d_{x}(h,h^*)-d_{x}(h,h^{*c}) \nonumber \\
&\ge d_{x'}(h,h^*_{DA})-d_{x}(h,h^*)-\left[ d_{x}(h,h^{*c}) + d_{x'}(h,h^{*c}_{DA}) \right] \nonumber 
\end{align}
Based on the above results and $d_{x’}(h,h^*_{DA})-d_{x}(h,h^*)\ge{\cal K}$, the proof is complete. 
\end{proof}

\begin{proof}{\textbf{of Theorem~\ref{theo5}.} }
Consider the following decomposition. For any $h \in {\cal H}$, 
\begin{align}
R(h)-R^c(h)&=R(h)+\hat R(h)-\hat R(h)+\hat R^c(h)-\hat R^c(h)-R^c(h) \nonumber \\
&=[R(h)-\hat R(h)]+[\hat R^c(h)-R^c(h)]+[\hat R(h)-\hat R^c(h)] .\nonumber 
\end{align}
Then, by Lemma~\ref{lemm3}, with probability at least $1-(\delta/7)$, we have 
\begin{align}
R(h)-R^c(h) \le 4M\sqrt{\frac{\ln(7/\delta)}{2n}} + [\hat R(h)-\hat R^c(h)] \nonumber 
\end{align}
Let $Z_i = [\hat R(h)-\hat R^c(h)]_i - [R(h)-R^c(h)]$, then it follows that $Z_i \in [-2M, 2M]$.  Hereafter, we simplify the expression of $Z_i$ stated in the previous proof, while preserving the same meaning. Similarly, by lemma~\ref{lemm3}, with probability at least $1-(\delta/7)$, we have 
\begin{align}
\hat R(h)-\hat R^c(h) &\le [R(h)-R^c(h)]+\sqrt{\frac{(4M)^2\ln(7/\delta)}{2n}} \nonumber \\
& \le \sigma\Delta_{\|\tilde y - y\|} + 4M\sqrt{\frac{(\ln(7/\delta)}{2n}}, \nonumber 
\end{align}
where the last inequality follows from Lemma~\ref{lemm2}. Then, we have 
\begin{align}
R(h)-R^c(h) \le \sigma\Delta_{\|\tilde y - y\|} + 8M\sqrt{\frac{(\ln(7/\delta)}{2n}}. \nonumber 
\end{align}
Consider the following decomposition. For any $h \in {\cal H}$,
\begin{align}
R^c(h)-R_{DA}(h)=&R^c(h)+\hat R^c(h)-\hat R^c(h)-\hat R(h) + [\hat R(h)-R_{DA}(h)] \nonumber \\
=&[R^c(h)-\hat R^c(h)]+[\hat R^c(h)-\hat R(h)] + [\hat R(h)-R_{DA}(h)] \nonumber \\
=&[R^c(h)-\hat R^c(h)]+[\hat R^c(h)-\hat R(h)]  \nonumber \\
&+[\hat R(h)+\hat R_{DA}(h) - \hat R_{DA}(h) - R_{DA}(h)] \nonumber \\
=&[R^c(h)-\hat R^c(h)]+[\hat R^c(h)-\hat R(h)]  \nonumber \\
& +[\hat R(h) - \hat R_{DA}(h)]+[\hat R_{DA}(h) - R_{DA}(h)] .\nonumber 
\end{align}
Then, by Lemma~\ref{lemm3} and the above, with probability at least $1-(\delta/7)$, we have 
\begin{align}
R^c(h)-R_{DA}(h) \le \sigma\Delta_{\|\tilde y - y\|} + 8M\sqrt{\frac{\ln(7/\delta)}{2n}} + [\hat R(h) - \hat R_{DA}(h)]. \nonumber 
\end{align}
Let $Z_i  = [\hat R(h) - \hat R_{DA}(h)]_i  - [R(h) - R_{DA}(h)]$, then it follows that $Z_i  \in [-2M, 2M]$. Similarly, by lemma~\ref{lemm3}, with probability at least $1-(\delta/7)$, we have 
\begin{align}
\hat R_{DA}(h)-\hat R(h) &\le [R_{DA}(h)-R(h)] + \sqrt{\frac{(4M)^2\ln(7/\delta)}{2n}} \nonumber \\
& \le \sigma\Delta_{\|x’-x\|} + \sqrt{\frac{(4M)^2\ln(7/\delta)}{2n}}, \nonumber 
\end{align}
where the last inequality follows from Lemma~\ref{lemm2}. Then we have 
\begin{align}
R^c(h)-R_{DA}(h) \le \sigma\Delta_{\|\tilde y - y\|} + \sigma\Delta_{\|x’-x\|} + 12M\sqrt{\frac{\ln(7/\delta)}{2n}} \nonumber 
\end{align}
If considering $R’_{DA}(h)$, $\sigma\Delta_{\|x’-x\|}$ can be replaced by $M\sqrt{2KL(P||P’_{DA})}$. Finally, let $ [R(h)-R^c(h)] + [R^c(h)-R_{DA}(h)] $, completing the proof. 
\end{proof} 

\begin{proof}{\textbf{of Corollary~\ref{coro5}.} }
According to Theorem~\ref{theo5}, we have 
\begin{align}
&R(h^*)-R_{DA}(h^*)\le \Omega, \nonumber \\
&R(h^*_{DA})-R_{DA}(h^*_{DA})\le \Omega, \nonumber 
\end{align}
where $\Omega$ denotes the RHS of the inequality in Theorem~\ref{theo5}. Moreover, by Definition~\ref{defi1}, it follows that $| R_{DA}(h^*) - R(h^*_{DA}) | \le \sigma d_{x’,x}$. The same result can also be derived by considering $R’_{DA}(h)$. Then, incorporating the above inequality completes the proof.
\end{proof} 

\begin{proof}{\textbf{of Theorem~\ref{theo8}.} }
Since $\ell$ is bounded by $M$, let $Z_{(s,i)}=\hat \gamma_{(s,i)}(h_1,h_2)-\gamma_s(h_1,h_2)$, we obtain that $Z_{(s,i)} \in [-2M,2M]$, where $s$ and $i$ denote the domain and sample indices, respectively. Then, by Lemma~\ref{lemm4}, for any $\varepsilon > 0$, we have 
\begin{align}
P\left\{ \frac{1}{E\bar n}\sum\limits_{s=1}^{E}{ \sum\limits_{i=1}^{\bar n}{\hat \gamma_{(s,i)}(h_1,h_2)} } - \frac{1}{E}\sum\limits_{s=1}^{E}{\gamma_s(h_1,h_2) }\ge \varepsilon \right\} \le exp\left\{ -\frac{2(E\bar n \varepsilon)^2}{E\bar n (4M)^2} \right\} . \nonumber 
\end{align}
Here, $\hat \gamma_{(s,i)}(h_1,h_2) = \ell(h_1,(x_i,y_i)_s) - \ell(h_2, (x_i,y_i)_s)$ and $\gamma_s(h_1,h_2) = {\mathbb E}_{P_s}[\ell(h_1,(x,y)_s) - \ell(h_2, (x,y)_s)]$.
\end{proof}

\begin{proof}{\textbf{of Corollary~\ref{coro8}.} }
Based on Theorem~\ref{theo8}, different types of inequalities can be derived. Similarly, we only illustrate how to decompose the relative expected risks, which as follows:
\begin{align}
{\mathbb E}[\ell(h_1,z)-\ell(h_2,z)]&={\mathbb E}[\ell(h_1,z)+\ell(h,z)-\ell(h,z)-\ell(h_2,z)] \nonumber \\
&={\mathbb E}[\ell(h_1,z)-\ell(h,z)]+{\mathbb E}[\ell(h,z)-\ell(h_2,z)]. \nonumber 
\end{align}
Considering the multi-domain settings, we have 
\begin{align}
\sum{\gamma_s(h,h_s)} = \sum{\gamma_s(h,h'_s)} + \sum{\gamma_s(h'_s, h_s)}.  \nonumber 
\end{align}
Note that, since $h$ is arbitrary, we allow $h_s$ and $h'_s$ to vary across domains. Then, based on Definition~\ref{defi1} and Definition~\ref{defi3}, we rewire ${\mathbb E}[\ell(h,z)-\ell(h_2,z)]$ to $\sigma d(h, h_2)$, completing the proof. 
\end{proof} 

\begin{proof}{\textbf{of Theorem~\ref{theo9}.} }
The proof follows that of Theorem~\ref{theo6}. Extending the argument to the multiple-domain setting completes the proof. 
\end{proof} 

\begin{proof}{\textbf{of Theorem~\ref{theo10}.} }
According to Lemma~\ref{lemm5}, for any $P_i$ and $Q_j$, we have 
\begin{align}
\left| {\mathbb E}_{P_i}[\ell]-{\mathbb E}_{Q_j}[\ell]\right| \le 2\sqrt{2}M d_{JS}(P_i,Q_j). \nonumber 
\end{align}
Then, extending this inequality to ${\bf P}$ and ${\bf P_t}$, we obtain 
\begin{align}
\left| \frac{1}{E_t}\sum\limits_{t=1}^{E_t}{ {\mathbb E}_{P_t}[\ell] } - \frac{1}{E}\sum\limits_{s=1}^{E}{ {\mathbb E}_{P_s}[\ell] }\right| \le C \sum\limits_{t=1}^{E_t}{\sum\limits_{s=1}^{E} d_{JS}(P_t,P_s)}, \nonumber 
\end{align}
where $C = (2\sqrt{2}M)/{E E_t}$ . 
Let $\lambda = \sum_t\sum_s d_{JS}(P_t,P_s)$ and 
\begin{align}
a&=\frac{1}{E_t}\sum_t{ {\mathbb E}_{z\sim P_t}[\ell(h, z) - \ell(h^{*c}_t, z)] } \nonumber \\
b&=\frac{1}{E_t}\sum_t{ {\mathbb E}_{z\sim P_t}[\ell(h^{*c}_t, z)] } \nonumber \\
c&=\frac{1}{E}\sum_s{ {\mathbb E}_{z\sim P_s}[\ell(h, z) - \ell(h^{*c}_s, z)] } \nonumber \\
d&=\frac{1}{E}\sum_s{ {\mathbb E}_{z\sim P_s}[\ell(h^{*c}_s, z)] }. \nonumber 
\end{align}
Applying the above inequality yields 
\begin{align}
\left| (a+b) - (c+d) \right| = \left| (a-c) - (d-b) \right| \le C{\lambda}.  \nonumber 
\end{align}
By the triangle inequality, we have 
\begin{align}
\left| a-c \right| - \left| d-b \right| &\le \left| (a-c) - (d-b) \right| \le C{\lambda}  \nonumber \\
\left| a-c \right| &\le \left| d-b \right| + C{\lambda} \nonumber 
\end{align}
Similarly, it follows from Lemma~\ref{lemm5} that 
\begin{align}
\left| d-b \right| = \left| \frac{1}{E}\sum_s{ {\mathbb E}_{z\sim P_s}[\ell(h^{*c}_s, z)] } - \frac{1}{E_t}\sum_t{ {\mathbb E}_{z\sim P_t}[\ell(h^{*c}_t, z)] } \right| \le C{\lambda}^{*c},\nonumber 
\end{align}
where ${\lambda}^{*c} = \sum_t\sum_s d_{JS}(P_{h^{*c}_t}, P_{h^{*c}_s})$. This implies that 
\begin{align}
\left| a-c \right| \le C{\lambda}^{*c} + C{\lambda}. \nonumber 
\end{align}
We now examine $|a-c|$ and consider 
\begin{align}
a &= \frac{1}{E_t}\sum_t{\mathbb E}[\ell(h,z)+\ell(h^*_t,z)-\ell(h^*_t,z)-\ell(h^{*c}_t,z)] \nonumber \\
&= \underbrace{\frac{1}{E_t}\sum_t{\mathbb E}[\ell(h,z) - \ell(h^*_t,z)]}_{a_1} - \underbrace{\frac{1}{E_t}\sum_t{\mathbb E}[\ell(h^{*c}_t,z) - \ell(h^*_t,z)] }_{a_2}\nonumber \\
c &= \frac{1}{E}\sum_s{\mathbb E}[\ell(h,z)+\ell(h^*_s,z)-\ell(h^*_s,z)-\ell(h^{*c}_s,z)] \nonumber  \\
& = \underbrace{\frac{1}{E}\sum_s{\mathbb E}[\ell(h,z) - \ell(h^*_s,z)]}_{c_1} - \underbrace{\frac{1}{E}\sum_s{\mathbb E}[\ell(h^{*c}_s,z) - \ell(h^*_s,z)]}_{c2}. \nonumber 
\end{align}
Similarly, applying the triangle inequality $| a-c |\ge | a_1 - a_2 | - | c_1 - c_2 | $, yields 
\begin{align}
\left| a_1 \right| \le \left| a_2 \right|+\left| c_1 \right|+\left| c_2 \right| + C{\lambda}^{*c} + C{\lambda}. \nonumber 
\end{align} 
Now, substituting $a_2=\sigma \varepsilon_{\eta_t}$ and $c_2=\sigma \varepsilon_{\eta}$ yields 
\begin{align}
\frac{1}{E_t}\sum_t{\mathbb E}[\ell(h,z) - \ell(h^*_t,z)] \le \sigma\varepsilon_{\eta_t} + \frac{1}{E}\sum_s{\mathbb E}[\ell(h,z) - \ell(h^*_s,z)]+ \sigma\varepsilon_{\eta}+ C{\lambda}^{*c} + C{\lambda}. \nonumber 
\end{align}
Now, by Lemma~\ref{lemm4}, the following inequality holds with probability at least $1-\delta$: 
\begin{align}
\frac{1}{E}\sum_s{\mathbb E}[\ell(h,z) - \ell(h^*_s,z)] \le \frac{1}{E \bar n}\sum_s{\hat \gamma(h,h^*_s)} + 4M\sqrt{\frac{\ln(1/\delta)}{2E\bar n}}, \nonumber 
\end{align}
which completes the proof. 
\end{proof}

\begin{proof}{\textbf{of Corollary~\ref{coro9}.} }
Based on the proof of Theorem~\ref{theo10}, let 
\begin{align}
a&=\frac{1}{E_t}\sum_t{ {\mathbb E}_{z\sim P_t}[\ell(h, z) - \ell(h^{*c}_t, z)] } \nonumber \\
b&=\frac{1}{E_t}\sum_t{ {\mathbb E}_{z\sim P_t}[\ell(h^{*c}_t, z)] } \nonumber \\
c&=\frac{1}{E}\sum_s{ {\mathbb E}_{z\sim P_s}[\ell(h, z) - \ell({h^{*c}_{DA}}_s, z)] } \nonumber \\
d&=\frac{1}{E}\sum_s{ {\mathbb E}_{z\sim P_s}[\ell({h^{*c}_{DA}}_s, z)] }. \nonumber 
\end{align}
Then, it similarly follows from Lemma~\ref{lemm5} that 
\begin{align}
\left| d-b \right| = \left| \frac{1}{E}\sum_s{ {\mathbb E}_{z\sim P_s}[\ell({h^{*c}_{DA}}_s, z)] } - \frac{1}{E_t}\sum_t{ {\mathbb E}_{z\sim P_t}[\ell(h^{*c}_t, z)] } \right| \le C{\lambda}^{*c}_{DA},\nonumber 
\end{align}
where ${\lambda}^{*c}_{DA} = \sum_t\sum_s d_{JS}(P_{h^{*c}_t}, P_{{h^{*c}_{DA}}_s})$. Then, we consider 
\begin{align}
c &= \frac{1}{E}\sum_s{\mathbb E}[\ell(h,z)+\ell({h^*_{DA}}_s,z)-\ell({h^*_{DA}}_s,z)-\ell({h^{*c}_{DA}}_s,z)]. \nonumber  \\
& = \underbrace{\frac{1}{E}\sum_s{\mathbb E}[\ell(h,z) - \ell({h^*_{DA}}_s,z)]}_{c_1} - \underbrace{\frac{1}{E}\sum_s{\mathbb E}[\ell({h^{*c}_{DA}}_s,z) - \ell({h^*_{DA}}_s,z)]}_{c2} . \nonumber 
\end{align}
Now, substituting $c_2=\sigma \varepsilon_{\eta_{DA}}$ yields 
\begin{align}
\frac{1}{E_t}\sum_t{\mathbb E}[\ell(h,z) - \ell(h^*_t,z)] \le& \sigma\varepsilon_{\eta_t} + \frac{1}{E}\sum_s{\mathbb E}[\ell(h,z) - \ell({h^*_{DA}}_s,z)] \nonumber \\
&+ \sigma\varepsilon_{\eta_{DA}}+ C{\lambda}^{*c}_{DA} + C{\lambda}_{DA}, \nonumber 
\end{align}
where, ${\lambda}_{DA} = \sum_t\sum_s d_{JS}(P_t,{P_{DA}}_s)$. Now, by Lemma~\ref{lemm4}, the following inequality holds with probability at least $1-\delta$: 
\begin{align}
\frac{1}{E}\sum_s{\mathbb E}[\ell(h,z) - \ell({h^*_{DA}}_s,z)] \le \frac{1}{E \bar n}\sum_s{\hat \gamma(h, {h^*_{DA}}_s)} + 4M\sqrt{\frac{\ln(1/\delta)}{2E\bar n}}, \nonumber 
\end{align}
which completes the proof. 
\end{proof}

\begin{proof}{\textbf{of Theorem~\ref{theo11}.} }
By Lemma~\ref{lemm6}, for any $i \ne j \in [E]$, we have 
\begin{align}
KL({P_{DA}}_i || {P_{DA}}_j) \le KL(P_i || P_j). \nonumber 
\end{align}
Let $M_{DA} = ({P_{DA}}_i + {P_{DA}}_j) / 2 $ and $M = (P_i+P_j)/2$. Then, still applying Lemma~\ref{lemm6}, we obtain
\begin{align}
KL({P_{DA}}_i || M_{DA}) &\le KL(P_i || M). \nonumber \\
KL({P_{DA}}_j || M_{DA}) &\le KL(P_j || M), \nonumber 
\end{align} 
which implies that 
\begin{align}
KL({P_{DA}}_i || M_{DA}) + KL({P_{DA}}_j || M_{DA}) &\le KL(P_i || M) + KL(P_j || M) \nonumber  \\ 
\frac{1}{2}\left( KL({P_{DA}}_i || M_{DA}) + KL({P_{DA}}_j || M_{DA}) \right) &\le \frac{1}{2}\left( KL(P_i || M) + KL(P_j || M) \right) . \nonumber 
\end{align}
Then, by Definition~\ref{defi5}, we have 
\begin{align}
JSD({P_{DA}}_1, \ldots, {P_{DA}}_E) &= \frac{1}{E}\sum_{i=1}^E{KL({P_{DA}}_i || M_{DA})} \nonumber  \\
JSD(P_1, \ldots, P_E) &= \frac{1}{E}\sum_{i=1}^E{KL(P_i || M)}, \nonumber 
\end{align}
where $M_{DA} = (1/E)\sum{{P_{DA}}_i}$ and $M = (1/E)\sum{P_i}$, completing the proof. 

\end{proof}

\begin{proof}{\textbf{of Theorem~\ref{theo12}.} }
By identifying $\{P_1,\ldots,P_E\}$ with $\{P_{h^*_1},\ldots,P_{h^*_E}\}$, estimating ${\bf 1}_{\{\phi(X)\in P_i\}} , \forall i \in [E]$, is equivalent to estimating ${\bf 1}_{\{\phi(X)\in P_{h^*_i}\}} , \forall i \in [E]$. Then, according to Definition~\ref{defi4}, $\phi \ne h^*_i$ implies 
\begin{align}
\phi(X) \ne h^*_i(X) \Leftrightarrow d(\phi,h^*_i) \ge \epsilon. \nonumber 
\end{align}
Then, by applying Lemma~\ref{lemm7}, the proof is complete. Note that we do not consider the case where the $n_i$ samples are i.i.d. observations drawn from domain $i$. Incorporating this setting would further refine the lower bound through its dependence on $n_i$. However, omitting this consideration does not affect the main insights of our analysis. 
\end{proof}

\begin{proof}{\textbf{of Theorem~\ref{theo13}.} }
According to Lemma~\ref{lemm5}, for any $P_i$ and $Q_j$, we have 
\begin{align}
\left| {\mathbb E}_{P_i}[\ell]-{\mathbb E}_{Q_j}[\ell]\right| \le 2\sqrt{2}M d_{JS}(P_i,Q_j). \nonumber 
\end{align}
Then, extending this inequality to ${\bf P}$ and ${\bf P_t}$, we obtain 
\begin{align}
\left| \frac{1}{E_t}\sum\limits_{t=1}^{E_t}{ {\mathbb E}_{P_t}[\ell] } - \frac{1}{E}\sum\limits_{s=1}^{E}{ {\mathbb E}_{P_s}[\ell] }\right| \le C \sum\limits_{t=1}^{E_t}{\sum\limits_{s=1}^{E} d_{JS}(P_t,P_s)}, \nonumber 
\end{align}
where $C = (2\sqrt{2}M)/{E E_t}$ .  
Let $\lambda = \sum_t\sum_s d_{JS}(P_t,P_s)$ and 
\begin{align}
a&=\frac{1}{E_t}\sum_t{ {\mathbb E}_{z\sim P_t}[\ell(h, z) - \ell(h^*_t, z)] } \nonumber \\
b&=\frac{1}{E_t}\sum_t{ {\mathbb E}_{z\sim P_t}[\ell(h^*_t, z)] } \nonumber \\
c&=\frac{1}{E}\sum_s{ {\mathbb E}_{z\sim P_s}[\ell(h, z) - \ell(h^*_s, z)] } \nonumber \\
d&=\frac{1}{E}\sum_s{ {\mathbb E}_{z\sim P_s}[\ell(h^*_s, z)] }. \nonumber 
\end{align}
Applying the above inequality and the triangle inequality yields that 
\begin{align}
\left| a-c \right| &\le \left| d-b \right| + C{\lambda}. \nonumber 
\end{align}
By lemma~\ref{lemm5}, we also have 
\begin{align}
\left| d-b \right| = \left| \frac{1}{E}\sum_s{ {\mathbb E}_{z\sim P_s}[\ell(h^*_s, z)] } - \frac{1}{E_t}\sum_t{ {\mathbb E}_{z\sim P_t}[\ell(h^*_t, z)] } \right| \le C{\lambda}. \nonumber 
\end{align}
This follows that 
\begin{align}
\left| a-c \right| \le 2 C \lambda, \nonumber 
\end{align}
that is 
\begin{align}
\frac{1}{E_t}\sum_t{\mathbb E}[\ell(h,z) - \ell(h^*_t,z)] \le  \frac{1}{E}\sum_s{\mathbb E}[\ell(h,z) - \ell(h^*_s,z)]+ 2C \lambda. \nonumber 
\end{align} 
Then, by Definition~\ref{defi1}, it follows that 
\begin{align}
\frac{1}{E}\sum_s{\mathbb E}[\ell(h,z) - \ell(h^*_s,z)] \le \sigma \frac{1}{E}\sum_s{d_{P_s}(h,h^*_s)}, \nonumber 
\end{align}
completing the proof. 
\end{proof}

{

\bibliography{sample}
}

\end{document}